\documentclass{article}
\PassOptionsToPackage{numbers}{natbib}
\usepackage[main,preprint]{neurips_2026}
\usepackage[utf8]{inputenc}
\usepackage[T1]{fontenc}
\usepackage{microtype}
\usepackage{url}
\pdfoutput=1

\usepackage{mathrsfs}
\usepackage{amsmath,amssymb}
\usepackage{bm}
\usepackage[usenames]{color}
\usepackage{amsthm}

\usepackage{multirow} 
\usepackage{enumitem}

\usepackage[colorlinks,
linkcolor=red,
anchorcolor=blue,
citecolor=blue
]{hyperref}

\usepackage{mylatexstyle}

\usepackage{graphicx,subcaption,wrapfig}

\newtheorem{theorem}{Theorem}

\newtheorem{lemma}{Lemma}

\newtheorem{definition}{Definition}

\usepackage{xcolor}

\usepackage[disable]{todonotes}

\usepackage{amssymb}
\usepackage{pifont}
\usepackage{xcolor}
\usepackage{booktabs}

\newcommand{\cmark}{\ding{51}} 
\newcommand{\xmark}{\ding{55}} 
\newcommand{\good}{\textcolor{green!70!black}{\cmark}}
\newcommand{\bad}{\textcolor{red!80!black}{\xmark}}

\newcommand{\wen}[1]{{\color{green}{\textbf{[Wenjing]} {#1}}}}

\title{Learning Theory of Transformers: Local-to-Global Approximation via Softmax Partition of Unity}

\author{%
  Zhongjie Shi \\
  School of Mathematics\\
  Georgia Institute of Technology\\
  Atlanta, GA 30332, United States\\
  \texttt{zshi332@gatech.edu}
  \And
  Wenjing Liao \\
  School of Mathematics\\
  Georgia Institute of Technology\\
  Atlanta, GA 30332, United States\\
  \texttt{wliao60@gatech.edu}
}
\date{}

\begin{document}

\maketitle

\begin{abstract}
This paper investigates the learning theory of Transformer networks for regression tasks on the compact Euclidean domain $[0,1]^d$ and $d$-dimensional compact Riemannian manifolds. We propose a novel constructive approximation framework for Transformers that builds local approximations of the target function and aggregates them into a global approximation via softmax partition of unity (POU).
This approach leverages the attention mechanism to achieve spatial localization  through affine transformations of the input. The softmax activation plays a crucial role in aggregating local approximations to a global output.
From an approximation perspective, we prove that a  dense Transformer equipped with only  two encoder blocks and standard single-hidden-layer point-wise feed-forward networks can achieve a uniform $\varepsilon$-approximation error for $\alpha$-H\"older continuous functions with $\alpha \in (0,1]$ using $\mathcal{O}(\varepsilon^{-d/\alpha})$ total parameters. Building upon this approximation guarantee, we establish a near minimax-optimal generalization error bound of order $\mathcal{O}\big(n^{-\frac{2\alpha}{2\alpha+d}} \log n\big)$ for the empirical risk minimizer, where $n$ is the training data size.  The Transformer architecture studied in this paper uses native softmax attention and sinusoidal positional encodings, with a fixed depth of two encoder blocks. Finally, numerical experiments reveal spatially localized, Gaussian-like attention weights that emerge during training, resembling our Softmax POU weights.

\end{abstract}

\section{Introduction}

During the past decade, deep learning has achieved remarkable success in various domains, including natural language processing \cite{graves2013speech, vaswani2017attention, liu2023pre}, computer vision \cite{krizhevsky2012imagenet, goodfellow2014generative, song2021scorebased}, and complex games \cite{silver2017mastering}. As a leading architecture, the Transformer \cite{vaswani2017attention} has become central to these advancements: Vision Transformers (ViTs) have surpassed classical convolutional networks in image recognition \cite{dosovitskiy2020image}, while large language models (LLMs) demonstrate exceptional generative and reasoning abilities \cite{brown2020language,ouyang2022training}. Despite this widespread empirical success, the theoretical understanding of the expressivity and generalization capability of the Transformers remains limited.

To establish a learning theory of Transformers, extensive research has investigated their expressivity and generalization. Initial studies focused on universal approximation properties, demonstrating that various Transformer models can universally approximate continuous (permutation equivariant) sequence-to-sequence functions \cite{yun2020transformers, yun2020n, zaheer2020big, kajitsuka2024are}. Beyond universality, inspired by the constructive techniques developed for classical deep ReLU networks \cite{yarotsky2017error, SchmidtHieber2020,shi2026can}, recent studies have shown that Transformers can achieve satisfying approximation bounds and has the potential to overcome the curse of dimensionality under certain structural assumptions \cite{gurevych2022rate, takakura2023approximation, jiang2024approximation, havrilla2024understanding, shi2025approximation, jiao2026transformers}. Based on these approximation results, further studies have provided generalization analysis and statistical rates. These works evaluate sample complexity and estimation errors in different theoretical settings \cite{zhang2022analysis, edelman2022inductive, gurevych2022rate, wei2022statistically, takakura2023approximation, havrilla2024understanding, shi2025learning, shi2025nonlinear, shi2025approximation}. However, despite these theoretical advances, a clear understanding of the specific role the softmax attention mechanism plays in the expressivity of transformer networks remains elusive, uncovering how this attention mechanism facilitates efficient feature representation learning is still an important problem.

A standard local-to-global approach to deep-ReLU-neural-network approximation on the unit cube $[0,1]^d$ covers the domain with small, geometrically simple patches, typically derived from a Cartesian grid or a triangulation, and combines local approximants through a partition of unity (POU) \cite{yarotsky2017error, yarotsky2018optimal, SchmidtHieber2020}. Depending on the target's smoothness, the local approximants are constants or Taylor polynomials; representative POU constructions use piecewise-linear spikes on triangulations \cite{yarotsky2018optimal} or tensor products of one-dimensional trapezoidal functions \cite{yarotsky2017error, SchmidtHieber2020}. However, localization in these constructions is tied to prescribed Cartesian grids or simplicial triangulations, whose adaptation to complex geometries such as compact manifolds requires additional geometric constructions. Chen et al.~\cite{chen2019efficient, chen2022nonparametric} establish manifold approximation results for deep ReLU networks using local coordinate charts and smooth POUs. Havrilla and Liao~\cite{havrilla2024understanding} extend this approach to Transformers with ReLU attention. However, it remains underexplored whether softmax attention can directly construct local-to-global approximations that adapt to complex domain geometries without explicit meshes or local coordinate charts.

We develop a mesh-free local-to-global constructive approximation framework based on Softmax POU, which provides a unified local-to-global approximation framework for Euclidean domains and manifolds, with potential extensions to more general geometries. Our method uses smooth, normalized weights to aggregate local approximations through native softmax attention. Specifically, to approximate a target function $g: \cX \rightarrow \RR$, we cover the input domain $\cX$ with $\ell_2$ balls of radius $r_g$ centered at a discrete set of points $\{\bc_i\}_{i=1}^{C_g}$. Given a scaling parameter $M_g > 0$, we define the spatially localized weights as
\begin{equation} \label{eq_softmax_weights}
    \beta_i(\bx) = \frac{\exp\left(M_g (r_g^2 - \|\bx - \bc_i\|_2^2)\right)}{\sum_{k=1}^{C_g} \exp\left(M_g (r_g^2 - \|\bx - \bc_k\|_2^2)\right)} = \frac{\exp\left(2M_g \langle \bx, \bc_i \rangle - M_g \|\bc_i\|_2^2\right)}{\sum_{k=1}^{C_g} \exp\left(2M_g \langle \bx, \bc_k \rangle - M_g \|\bc_k\|_2^2\right)}.
\end{equation}
The target function $g$ is then approximated by the convex combination of its values at these discrete centers: $\hat g(\bx) = \sum_{i=1}^{C_g} \beta_i(\bx) g(\bc_i)$. This Softmax POU has three properties relevant to the Transformer construction. Geometrically, its weights use Euclidean distances to covering centers without requiring an explicit Cartesian grid, allowing the construction to be applied on manifolds. The normalized weight at each center depends on the distances to all centers, coupling their contributions through a shared denominator. Algebraically, softmax enforces the exact POU condition $\sum_{i=1}^{C_g} \beta_i(\bx) = 1$, while the common input-dependent quadratic term $\|\bx\|_2^2$ cancels in the normalization. The resulting scores are affine in $\bx$ and can be realized through native dot-product attention. These properties provide a direct correspondence between the local-to-global approximation scheme and the Transformer construction.

Building upon these insights, this paper provides a comprehensive theoretical analysis of the expressive power and generalization capabilities of Transformers. Our main contributions are summarized as follows

\begin{itemize}
    \item \textbf{Softmax POU Approximation Framework.} We propose a novel local-to-global constructive approximation framework based on a globally supported Softmax POU. By algebraically canceling input-dependent quadratic terms, this mesh-free POU mechanism translates relative Euclidean distances into purely affine transformations, inherently establishing a dynamic global attention competition to achieve spatial localization. This exact structural correspondence empowers shallow Transformers to seamlessly fuse localized approximations into an aggregated global approximation via native softmax dot-product attention mechanism.

    \item \textbf{Learning on Euclidean Domains via Softmax POU.} Applying this framework to compact Euclidean spaces $[0,1]^d$, we constructively prove that a shallow (only two encoder blocks with standard single-hidden-layer FFNs), wide, and dense Transformer with softmax activation and sinusoidal positional encodings achieves a uniform $\varepsilon$-approximation error for $\alpha$-H\"older continuous functions using $\mathcal{O}(\varepsilon^{-d/\alpha})$ parameters (see Theorem~\ref{thm_holder_approx}). Building upon this, we establish a near minimax-optimal convergence rate of $\mathcal{O}\big(n^{-\frac{2\alpha}{2\alpha+d}} \log n\big)$ (see Theorem~\ref{thm_gen_cube}), formally validating the inherent statistical efficiency of the native Transformer architecture without relying on artificial sparsity or structural depth constraints.
    
    \item \textbf{Adaptivity to Riemannian Manifolds.} We extend our theoretical guarantees to target functions defined on $d$-dimensional compact Riemannian manifolds embedded in $\RR^{\bar{d}}$. We formally prove that by operating directly on the ambient Euclidean metric, the Softmax POU framework natively adapts to such complex geometries. Consequently, shallow Transformers can achieve a uniform $\varepsilon$-approximation error using $\mathcal{O}(\varepsilon^{-d/\alpha})$ parameters (see Theorem~\ref{thm_manifold_approx}) and a near minimax-optimal convergence rate of $\mathcal{O}\big(n^{-\frac{2\alpha}{2\alpha+d}} \log n\big)$ (see Theorem \ref{thm_gen_manifold}). Because these bounds scale exclusively with the intrinsic dimension $d$ rather than the ambient dimension $\bar{d}$, our analysis establishes the native Transformer architecture's inherent capacity to adapt to complex geometries and overcome the curse of dimensionality.
\end{itemize}

The remainder of this paper is organized as follows: Section~\ref{sec:architecture} defines the Transformer architecture and its hypothesis space. Section~\ref{sec:euclidean} establishes the theoretical guarantees on compact Euclidean domains: Subsection~\ref{sec:POU_framework} introduces the Softmax POU framework to prove uniform approximation bounds, and Subsection~\ref{sec:generalization} leverages these results to derive generalization rates. Section~\ref{sec:mtor} extends this framework to compact Riemannian manifolds, demonstrating the architecture's adaptivity to complex geometries and its ability to avoid the curse of dimensionality. Section~\ref{sec:numerical} presents numerical experiments providing empirical support for the spatial localization mechanism underlying our construction. Section~\ref{sec:related_work} discusses related literature, and Section~\ref{sec:conclusion} concludes the paper. Appendices~\ref{sec:appendixA}--\ref{sec:appendixC} provide detailed proofs, and Appendix~\ref{app:numerical} provides experimental details and additional results.

\paragraph{Notation.} We use lower case bold letters for vectors ($\bx, \by$), upper case bold for matrices ($\bX, \bA$), and calligraphic letters for sets and spaces ($\cM, \cH$). For a vector $\bx$, $\|\bx\|_p$ denotes its standard $\ell_p$ norm, and $\langle \bx, \by \rangle$ denotes the Euclidean inner product. For a matrix $\bX$, $\mathrm{vec}(\bX)$ denotes its column-wise vectorization, $\|\bX\|_{\max} = \max_{i,j} |\bX_{i,j}|$ denotes the maximum absolute value among its elements. For a function $f: \cX \to \RR$, $\|f\|_\infty := \sup_{\bx \in \cX} |f(\bx)|$ denotes its supremum norm. Let $\mathcal{B}(\bc, r)$ and $\bar{\mathcal{B}}(\bc, r)$ denote the open and closed Euclidean balls centered at $\bc$ with radius $r$, respectively. We define $[n] := \{1, \dots, n\}$. The Kronecker delta $\delta_{i,j}$ evaluates to $1$ if $i=j$ and $0$ otherwise.


\section{Transformer architecture} \label{sec:architecture}

In this section, we formally introduce the structure of the Transformer model. The architecture consists of a pre-processing stage followed by $L$ blocks of Transformer encoders. Each encoder block comprises a multi-head attention (MHA) layer, followed by a point-wise FFN layer.

For an input vector $\bx \in \RR^d$, the pre-processing step $\cP: \RR^d \to \RR^{D \times P}$ converts the input to an initial hidden representation $\bZ_0 \in \RR^{D \times P}$ through an affine embedding and the addition of a positional encoding (PE) $\bP \in \RR^{D \times P}$. Formally, the pre-processing operation is defined as
\begin{equation*}
    \bZ_0 = \cP(\bx) = (\bW_E \bx + \bb_E) \be_1^\top + \bP,
\end{equation*}
where $\bW_E \in \RR^{D \times d}$ is the embedding weight matrix, $\bb_E \in \RR^D$ is the bias vector, and $\be_1 = [1, 0, \dots, 0]^\top \in \RR^P$ is the canonical basis vector that assigns the embedded input exclusively to the first token position. Here, $P \ge 1$ denotes the total sequence length, implying that the remaining $P - 1$ columns of $\bZ_0$ serve as appended padding or latent tokens. We use sinusoidal positional encodings in our constructions, which aligns with the Transformer architectures in practical applications \cite{vaswani2017attention}.  

Consider the $i$-th encoder block with input $\bZ_{i-1} = [\bz^1_{i-1}, \dots, \bz^P_{i-1}] \in \RR^{D \times P}$. Within the MHA layer $\cA_i: \RR^{D \times P} \to \RR^{D \times P}$ with $H_i$ heads, each head includes a query matrix $\bQ_i^h \in \RR^{d_k \times D}$, a key matrix $\bK_i^h \in \RR^{d_k \times D}$, and a value matrix $\bV_i^h \in \RR^{d_v \times D}$, where $d_k$ is the uniform query/key dimension and $d_v$ is the uniform value dimension. For $h\in [H_i]$, the output of the $h$-th attention head ${\head}^h_i \in \RR^{d_v \times P}$ is computed as
\begin{equation*}
    {\head}^h_i =  \bV_i^h \bZ_{i-1} \bA_i^h,
\end{equation*}
where $\bA_i^h \in \RR^{P \times P}$ is the attention probability matrix. Its $j$-th column is defined by
\begin{equation*}
    (\bA_i^h)_{:, j} := \softmax \left( \left(\bZ_{i-1}^\top {\bK_i^h}^\top \bQ_i^h \bZ_{i-1}\right)_{:, j} \right), \quad j \in [P],
\end{equation*}
with the $\softmax: \RR^{P} \to \RR^{P}$ operator applied component-wise as
\begin{equation*}
    \softmax(\bx) = \left[ \frac{e^{x_1}}{\sum_{l=1}^P e^{x_l}}, \dots, \frac{e^{x_{P}}}{\sum_{l=1}^P e^{x_l}} \right]^\top.
\end{equation*}
To preserve the distinct representations from different heads, the final output of the MHA layer $\widehat{\bZ}_i \in \RR^{D \times P}$ is generated by concatenating the outputs of all $H_i$ heads along the feature dimension, followed by a linear projection
\begin{equation*}
    \widehat{\bZ}_i = \bW^O_i \begin{bmatrix} {\head}^1_i \\ \vdots \\ {\head}^{H_i}_i \end{bmatrix},
\end{equation*}
where $\bW^O_i \in \RR^{D \times (H_i \cdot d_v)}$ is the output projection matrix.

The output $\widehat{\bZ}_i= [\hat \bz_i^1, \dots, \hat \bz^{P}_i] \in \RR^{D \times P}$ is subsequently fed into a point-wise FFN. The FFN layer $\cF_i: \RR^{D \times P} \to \RR^{D \times P}$ with weight matrices $\bW_i^1 \in \RR^{d_{\text{ff}} \times D}$, $\bW_i^2 \in \RR^{D \times d_{\text{ff}}}$ and biases $\bb_i^1 \in \RR^{d_{\text{ff}}}$, $\bb_i^2 \in \RR^{D}$ computes
\begin{equation} \label{FFNlayer}
    \bz^j_i =  \bW_i^2 \sigma\left(\bW_i^1 \hat \bz^j_i +\bb_i^1 \right) +\bb_i^2, \quad   j \in [P],
\end{equation}
where $\sigma(\cdot) = \max(0, \cdot)$ is the ReLU activation function applied component-wise.

Finally, for the Transformer model with $L$ encoder blocks, its output is
\begin{equation*}
    \cT_L(\bx) = \cF_L \circ \cA_L \circ \cdots \circ \cF_1 \circ \cA_1 \circ \cP(\bx) \in \RR^{D \times P}.
\end{equation*}
The scalar output is then obtained by reading out the relevant components of $\cT_L(\bx)$, via a linear affine mapping $\bc_{L+1}^\top \mathrm{vec}(\cT_L(\bx))$. With the complete architecture established, we now define the function class expressed by such Transformer networks.

\begin{definition}[Transformer Network Class] \label{def_transformer_class}
    For depth $L \in \NN$, embedding dimension $D \in \NN$, sequence length $P \in \NN$, layer configurations $\{H^\ell\}_{\ell=1}^L, \{d_k^\ell\}_{\ell=1}^L, \{d_v^\ell\}_{\ell=1}^L, \{d_{\text{ff}}^\ell\}_{\ell=1}^L \subset \NN$, and parameter magnitude bound $M > 0$, we define the class of Transformer networks as
    \begin{align*}
        &\mathcal{T}\Big(L, D, P, \{H^\ell\}_{\ell=1}^L, \{d_k^\ell\}_{\ell=1}^L, \{d_v^\ell\}_{\ell=1}^L, \{d_{\text{ff}}^\ell\}_{\ell=1}^L, M\Big) \\
        &\quad = \left\{ f_\theta \;\left|\; 
        \begin{aligned}
            &f_\theta(\bx)= \bc_{L+1}^\top \mathrm{vec}(\cT_L(\bx)) \text{ is an } L\text{-block Transformer with embedding dim } D,  \\
            &\text{ sequence length } P, \text{parameter bound } \|\theta\|_\infty \le M, \text{ and for each block } \ell \in [L]:  \\
            & \text{there is } H^\ell \text{heads}, \text{query/key dim is } d_k^\ell, \text{ value dim is } d_v^\ell, \text{ FFN hidden width is } d_{\text{ff}}^\ell
        \end{aligned}
        \right. \right\}.
    \end{align*}
\end{definition}

\section{Learning on Euclidean domains} \label{sec:euclidean}

In this section, we establish the theoretical capabilities of the Transformer architecture. We focus our analysis on compact Euclidean domains and introduce a novel Softmax POU framework. In Subsection~\ref{sec:POU_framework}, we constructively prove uniform approximation rates for H\"older continuous functions. Building upon these approximation results, we further establish the corresponding statistical generalization bounds in Subsection~\ref{sec:generalization}.

\subsection{Approximation of H\"older functions on Euclidean domains} \label{sec:POU_framework}

A foundational ingredient in our construction is the Softmax POU. Unlike deep ReLU networks that rely on locally compactly supported bump functions as POU \cite{yarotsky2017error, yarotsky2018optimal, SchmidtHieber2020}, we introduce a globally supported "soft" POU by Softmax POU. This idea shares connections with classical non-parametric techniques, such as Shepard's method for spatial interpolation \cite{shepard1968two} and normalized Radial Basis Function (RBF) networks \cite{moody1989fast}. By utilizing a scaling parameter $M_g$ to control the exponential decay outside local neighborhoods, this scheme naturally establishes a local-to-global approximation paradigm. We formalize this mathematical construction in the following lemma, which is proved in Appendix \ref{sec:appendixA1}.

\begin{lemma}[Softmax POU Approximation on Euclidean Space] \label{lemma_POU_approx}
    Let $d \in \NN$, $\alpha \in (0,1]$, and $g: [0,1]^d \to \RR$ be an $\alpha$-H\"older continuous function with constant $C_H$. For any target accuracy $\varepsilon \in (0,\frac{1}{e}]$, let the covering radius be $r_g = (\varepsilon / (4C_H))^{1/\alpha}$. There exists a finite set of $C_g$ centers $\{\bc_i\}_{i=1}^{C_g} \subset [0,1]^d$ forming an $r_g$-covering of the domain (i.e., $[0,1]^d \subseteq \bigcup_{i=1}^{C_g} \bar{\mathcal{B}}(\bc_i, r_g)$) with $C_g \le \left( \sqrt{d} (4C_H)^{1/\alpha}\right)^d \varepsilon^{-\frac{d}{\alpha}}$. By setting the scaling parameter 
    \begin{equation*}
        M_g = \frac{(4C_H)^{2/\alpha}}{3} \left(\log\left( 4\|g\|_\infty \left( \sqrt{d} (4C_H)^{1/\alpha}\right)^d \right)+ \frac{d+\alpha}{\alpha} \right) \varepsilon^{-\frac{2}{\alpha}} \log \frac{1}{\varepsilon} ,
    \end{equation*}
    the Softmax POU approximation
    $\hat g(\bx) = \sum_{i=1}^{C_g} \beta_i(\bx) g(\bc_i),$
    with
    \begin{align}
        \beta_i(\bx) = \frac{\exp\left(M_g (r_g^2 - \|\bx - \bc_i\|_2^2)\right)}{\sum_{k=1}^{C_g} \exp\left(M_g (r_g^2 - \|\bx - \bc_k\|_2^2)\right)}= \frac{\exp\left(2M_g \inner{\bx}{\bc_i} - M_g \|\bc_i\|_2^2\right)}{\sum_{k=1}^{C_g} \exp\left(2M_g \inner{\bx}{\bc_k} - M_g \|\bc_k\|_2^2\right)},
        \label{eq:lemma1betai}
    \end{align}
    achieves the uniform error bound
    $\sup_{\bx \in [0,1]^d} |\hat g(\bx) - g(\bx)| \le \varepsilon$.
\end{lemma}

\begin{wrapfigure}{r}{0.55\textwidth}    \centering
    \begin{subfigure}[b]{0.24\textwidth}
        \centering
        \includegraphics[width=\textwidth,height=2.8cm]{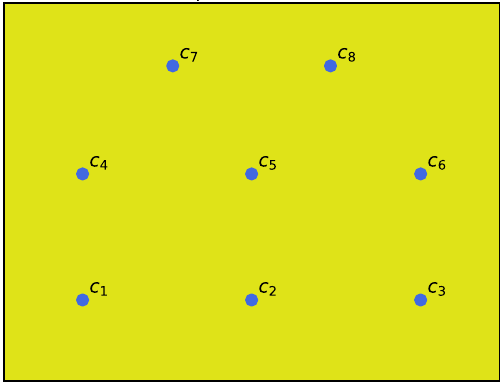}
        \caption{Centers}
    \end{subfigure}
    \hspace{0.2cm}
    \begin{subfigure}[b]{0.24\textwidth}
        \centering     \includegraphics[width=\textwidth,height=2.82cm]{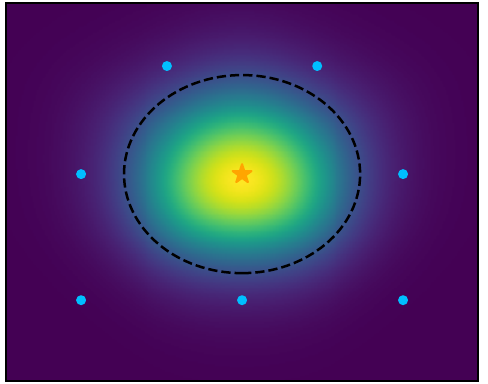}
        \caption{$\beta_5(\bx)$}
    \end{subfigure}
    \caption{Geometric illustration of the softmax partition of unity. The centers $\{\bc_i\}_{i=1}^{C_g}$ and covering balls $\bar{\mathcal{B}}(\bc_i,r_g)$ induce the localized soft  weights $\beta_i(\bx)$.}
    \label{fig:softpou}
        \vspace{-0.4cm}
\end{wrapfigure}

Lemma \ref{lemma_POU_approx} introduces our new approximation framework. The centers $\{\bc_i\}_{i=1}^{C_g}$ and the corresponding covering balls $\bar{\mathcal{B}}(\bc_i,r_g)$ induce localized soft regions through the weights $\beta_i(\bx)$, which is illustrated in Figure \ref{fig:softpou}. The weight functions $\{\beta_i(\bx)\}_{i=1}^{C_g}$ form a partition of unity over the domain, and the scaling parameter $M_g$ controls the exponential decay of $\beta_i(\bx)$.

The Softmax POU approximation scheme offers two primary advantages: (i) the algebraic cancellation of the input-dependent quadratic term $\|\bx\|_2^2$ eliminates the geometric quadratic form, reducing the relative distance computation to purely affine transformations; and (ii) it perfectly aligns with the native Transformer architecture, enabling mesh-free spatial localization that is dynamically achieved through standard softmax dot-product attention mechanism.
The following theorem quantifies the approximation power of Transformers using the Softmax POU approximation framework.

\begin{theorem}[Transformer Approximation on Euclidean Space] \label{thm_holder_approx}
    Let $d \in \NN$, $\alpha \in (0,1]$, and $g: [0,1]^d \to \RR$ be $\alpha$-H\"older continuous with constant $C_H$ and $\|g\|_\infty \le B$. For any $\varepsilon \in (0, 1/e]$, there exists a Transformer network 
    \begin{equation*}
        \hat{g}_T \in \mathcal{T}\Big(L, D, P, \{H^\ell\}_{\ell=1}^L, \{d_k^\ell\}_{\ell=1}^L, \{d_v^\ell\}_{\ell=1}^L, \{d_{\text{ff}}^\ell\}_{\ell=1}^L, M_{max}\Big)
    \end{equation*}
    satisfying 
    $$\sup_{\bx \in [0,1]^d} |\hat{g}_T(\bx) - g(\bx)| \le \varepsilon$$
    with the following parameters
    \begin{itemize}
        \item $L=2$, $D=d+4$, and $P \le C_P \varepsilon^{-d/\alpha}$ where $C_P = d^{d/2} (8C_H)^{d/\alpha}$.
        \item For $\ell=1$: $H^1 = P+2$, $d_k^1 = d_v^1 = 2$, and $d_{\text{ff}}^1 = 2d+8$.
        \item For $\ell=2$: $H^2 = 1$, $d_k^2 = d_v^2 = 1$, and $d_{\text{ff}}^2 = 2d+8$.
    \end{itemize}
    The parameter magnitude $M_{max}$ and total dense parameters $\cN_{total}$ satisfy
    \begin{align*}
        M_{max} \le \widetilde{C}_{mag} \varepsilon^{-2d/\alpha} \log \frac{1}{\varepsilon}, \qquad \cN_{total} \le C_N \varepsilon^{-d/\alpha},
    \end{align*}
    where 
    $C_N= 10(d+4) C_P + 9d^2 + 95d + 236$, and
    $\widetilde{C}_{mag} = \max\left\{ \frac{1}{1 - \frac{1}{eB}}, C_{mag}, 2 \left(1 + \frac{1}{2eB}\right) C_B \right\},$
    with $C_B = \max\{3d C_M, B, 1\}$, $C_{mag} = \frac{C_P^2 C_{log}}{8}$, $C_{log} = \left| \log(4 C_P B (1 + 6 d C_M)) \right| + \frac{d+2+\alpha}{\alpha} + 2$, and $C_M = \frac{(8C_H)^{2/\alpha}}{3} \left(\log\big( 4B d^{d/2}(4C_H)^{d/\alpha} \big) + \frac{d+\alpha}{\alpha} \right)$.
\end{theorem}

Theorem~\ref{thm_holder_approx} places representational capacity in the sequence dimension $P$ while retaining two dense encoder blocks and single-hidden-layer point-wise FFNs. Native dot-product softmax attention supplies spatial localization. Sinusoidal positional encodings on a uniform angular grid compute relative positions through phase differences; their cyclic symmetries preserve token identities across layers.

\begin{proof}[Proof Sketch of Theorem \ref{thm_holder_approx}]
    The detailed proof of Theorem \ref{thm_holder_approx} is deferred to Appendix~\ref{sec:appendixB1}. We construct a 2-block Transformer to execute the Softmax POU approximation scheme by leveraging the sequence dimension $P$ as an index set to parallelize the localized expert approximations. The procedure is organized into four key stages:

    \textbf{Pre-processing (see Lemma \ref{lemma_preprocessing}):} The input $\bx \in [0,1]^d$ is mapped to an initial sequence matrix $\bZ_0 \in \RR^{D \times P}$. This layer embeds $\bx$ and an indicator exclusively into the first token, zero-pads the remaining sequence, and appends sinusoidal positional encodings across all $P$ tokens
    \begin{equation*}
        (\bZ_0)_{:,j} = \begin{bmatrix} \delta_{j,1} \bx^\top, & \delta_{j,1}, & 0, & \sin(\theta_j), & \cos(\theta_j) \end{bmatrix}^\top.
    \end{equation*}
    \textbf{First MHA (Local Feature Extraction, see Lemma \ref{lemma_mha_affine}):} Utilizing $H^1 = P+2$ heads, this layer computes parallel affine transformations to extract local spatial features and approximation expert features. The output $\widehat{\bZ}_1 = \cA_1(\bZ_0)$ takes the column-wise form
    \begin{equation*}
        (\widehat{\bZ}_1)_{:,j} = \begin{bmatrix} \tilde{T}_j(\bx), & \tilde{g}(\bc_j), & \bm{0}_{d-1}^\top, & \tilde{I}_j, & \lambda(M)\sin(\theta_j), & \lambda(M)\cos(\theta_j) \end{bmatrix}^\top,
    \end{equation*}
    where $\tilde{T}_j(\bx) \approx 2M_g \bc_j^\top \bx - M_g\|\bc_j\|_2^2$ captures the approximated local spatial feature, $\tilde{g}(\bc_j) \approx g(\bc_j)$ represents the approximated local expert feature for each center $\bc_j$, and $\tilde{I}_j \approx \delta_{j,1}$ serves as the approximated structural indicator for the first token. Crucially, in our construction, the exact angular identity $\theta_j$ of the sinusoidal positional encodings remains invariant up to a uniform scalar scaling $\lambda(M)$.

    \textbf{First FFN (Non-linear Restoration, see Lemma \ref{lemma_FFN_restore}):} By exploiting the hard-thresholding gating property of the ReLU activation, $\bZ_1 = \cF_1(\widehat{\bZ}_1)$ functions as a non-linear filter to eliminate the small error introduced by the softmax operation, thereby exactly recovering the canonical form and original sinusoidal positional encodings
    \begin{equation*}
        (\bZ_1)_{:,j} = \begin{bmatrix} \tilde{T}_j(\bx), & \tilde{g}(\bc_j), & \bm{0}_{d-1}^\top, & \delta_{j,1}, & \sin(\theta_j), & \cos(\theta_j) \end{bmatrix}^\top.
    \end{equation*}
    \textbf{Second MHA (Global Aggregation, see Lemma \ref{lemma_mha_pou}):} A single head ($H^2 = 1$) extracts the indicator $\delta_{k,1}$ via the Query, the local spatial features $\tilde{T}_j(\bx)$ via the Key, and the local expert features $\tilde{g}(\bc_j)$ via the Value. The softmax activation then directly computes the normalized POU weights $\tilde{\beta}_j(\bx) = \exp(\tilde{T}_j(\bx)) / \sum_{l=1}^P \exp(\tilde{T}_l(\bx))$, aggregating the local expert predictions into a unified global approximation $\hat{g}_T(\bx)$ and storing the result in the first column of $\widehat{\bZ}_2 = \cA_2(\bZ_1)$, that is
    \begin{equation*}
        (\widehat{\bZ}_2)_{:,1} = \begin{bmatrix} \hat{g}_T(\bx), & \bm{0}_{D-1}^\top \end{bmatrix}^\top, \quad \text{where} \quad \hat{g}_T(\bx) = \sum_{j=1}^P \tilde{\beta}_j(\bx)\tilde{g}(\bc_j).
    \end{equation*}
    \textbf{Second FFN \& Readout:} Finally, $\cF_2$ operates as an exact identity mapping $\bZ_2 = \widehat{\bZ}_2$. A linear readout vector $\bc^\top$ explicitly extracts the final approximation $\hat{g}_T(\bx)$ from the hidden state.
\end{proof}

\subsection{Generalization analysis} \label{sec:generalization}

In this section, we analyze the generalization capability of the empirical risk minimization (ERM) algorithm over the bounded hypothesis space constructed based on our Transformer architecture. We utilize the covering number to measure the capacity of the hypothesis space, which subsequently provides insights into the learning performance of the ERM algorithm.

We follow the classical learning framework for regression \cite{cucker2007learning}. Suppose that a data sample $\mathcal{S} = \{(\bx_i, y_i)\}_{i=1}^n \subset \cX \times \mathcal{Y}$ is independently drawn from a true unknown Borel probability measure $\rho$ on $\cX \times \mathcal{Y}$, with $\mathcal{Y} \subseteq [-B, B]$ for some $B > 0$. The target function for learning is the regression function $f_\rho: \cX \to \RR$ defined by $f_\rho(\bx) = \int_{\mathcal{Y}} y d\rho(y|\bx)$, which minimizes the expected generalization error
\begin{equation*}
    \mathcal{E}(f) := \int_{\cX \times \mathcal{Y}} (f(\bx) - y)^2 d\rho.
\end{equation*}
We denote by $\rho_{\mathcal{X}}$ the marginal distribution of $\rho$ on $\cX$, and by $(L_{\rho_{\mathcal{X}}}^2, \|\cdot\|_\rho)$ the Hilbert space of square-integrable functions with respect to $\rho_{\mathcal{X}}$, where $\|f\|_\rho^2 := \int_{\cX}|f(\bx)|^2\,d\rho_{\mathcal{X}}(\bx)$. For every $f\in L_{\rho_{\mathcal{X}}}^2$, the squared-loss excess risk satisfies $\mathcal{E}(f)-\mathcal{E}(f_\rho)=\|f-f_\rho\|_\rho^2$.

We consider the ERM algorithm over the Transformer hypothesis space $\mathcal{T}$ defined in Theorem \ref{thm_holder_approx}. 
Based on this hypothesis space, the empirical target function is defined as
\begin{equation*}
    f_{\mathcal{S}} := \arg\min_{f \in \mathcal{T}} \frac{1}{n} \sum_{i=1}^n (f(\bx_i) - y_i)^2.
\end{equation*}
Since $\mathcal{Y} \subseteq [-B, B]$, we project the output function onto the interval $[-B, B]$ and define the truncated empirical target function as $\pi_B f_{\mathcal{S}}$, where $\pi_B(t) = \max\{-B, \min\{t, B\}\} = -B+\operatorname{ReLU}(t+B)-\operatorname{ReLU}(t-B)$. Thus, clipping is implemented exactly by one hidden ReLU layer with two units and an affine readout.

For target functions defined on Euclidean domains, based on the approximation error bound in Theorem \ref{thm_holder_approx} and the covering number bound in Lemma \ref{lemma_covering_number} in Appendix \ref{sec:covering}, we obtain the following near minimax-optimal convergence rate of the Transformer estimator $f_{\mathcal{S}}$.

\begin{theorem} \label{thm_gen_cube}
    Suppose that the regression function $f_\rho: [0,1]^{d} \to \RR$ is an $\alpha$-H\"older continuous function with constant $C_H$ and $\|f_\rho\|_\infty \le B$. Consider the truncated empirical target function $\pi_B f_{\mathcal{S}}$ over the hypothesis space $\mathcal{T}$ defined in Theorem \ref{thm_holder_approx}, for $0 < \varepsilon \le 1/e$, we have
    \begin{equation*}
        \mathbb{E}\left[\left\|\pi_B f_{\mathcal{S}}-f_\rho\right\|_\rho^2\right] \le C_4 \max\left\{ \varepsilon^2, \frac{\varepsilon^{-d/\alpha} \log \frac{1}{\varepsilon}}{n} \right\}.
    \end{equation*}
    Moreover, by choosing $\varepsilon = n^{-\frac{\alpha}{2\alpha+d}}$, we obtain
    \begin{equation*}
        \mathbb{E}\left[\left\|\pi_B f_{\mathcal{S}}-f_\rho\right\|_\rho^2\right] \le C_4 n^{-\frac{2\alpha}{2\alpha+d}} \log n,
    \end{equation*}
    where $C_4$ is a positive constant depending on $d, \alpha, B$, and $C_H$ specified in \eqref{C4}.
\end{theorem}

\begin{proof}[Proof Sketch of Theorem \ref{thm_gen_cube}]
    The detailed proof is deferred to Appendix \ref{sec:appendixC}. The result follows from optimizing the bias-variance trade-off established via an oracle inequality (Lemma \ref{lemma_oracle}). First, we bound the approximation error (bias) by $\varepsilon$ utilizing the Transformer construction from Theorem \ref{thm_holder_approx}. Second, we control the estimation error (variance) by bounding the capacity of the hypothesis space $\mathcal{T}$ as detailed in Appendix \ref{sec:covering}. Specifically, Lemma \ref{lemma_covering_number} provides the covering number bound
    \begin{equation*}
        \log \mathcal{N}(\eta, \mathcal{T}, \|\cdot\|_\infty) \le \mathcal{N}_{total} \log \left( \frac{1224256 P^4 D^{22} M_{max}^{26}}{\eta} \right).
    \end{equation*}
    Finally, balancing the approximation and estimation errors and choosing the optimal network accuracy $\varepsilon = n^{-\frac{\alpha}{2\alpha+d}}$ yields the desired convergence rate as detailed in Appendix \ref{sec:learningrate}.
\end{proof}

\section{Adaptivity to Riemannian manifolds} \label{sec:mtor}

In this section, we assume that the input domain $\cM \subseteq [0,1]^{\bar d}$ is a compact and connected $d$-dimensional Riemannian manifold with strictly positive reach $\tau > 0$ (see Definition \ref{def.reach} in Appendix \ref{sec:appendixA2}). Our goal is to investigate the theoretical capabilities of Transformers for learning $\alpha$-H\"older continuous functions defined on $\cM$ (see Definition \ref{def_holder} in Appendix \ref{sec:appendixA2}). More details about manifolds can be found in \cite{tu2011manifolds,lee2006riemannian}. 

Our Softmax POU approximation scheme naturally adapts to complex geometries by operating directly on the ambient Euclidean metric. Because the ReLU activation is fundamentally piecewise linear, locally supported ReLU-based POUs are geometrically constrained to partition space using hyperplanes, such as explicit Cartesian grids or simplicial triangulations. In contrast, the Softmax POU is inherently mesh-free. By exploiting the local metric equivalence between intrinsic geodesic and ambient Euclidean distances, it achieves spatial localization purely through affine transformations, avoiding the need for explicit surface parameterizations or complex non-linear geodesic computations. This exact formulation perfectly aligns with the native dot-product attention mechanism in Transformers. Building upon this geometric intuition, the following theorem formally establishes how Transformers natively execute this local-to-global approximation scheme to bypass the geometric complexities of compact Riemannian manifolds.

\begin{theorem}[Transformer Approximation on Manifolds] \label{thm_manifold_approx}
    Let $\cM \subseteq [0,1]^{\bar{d}}$ be a $d$-dimensional compact and connected Riemannian manifold with reach $\tau > 0$, and let $g: \cM \to \RR$ be an $\alpha$-H\"older continuous function with constant $C_H$ and $\|g\|_\infty \le B$. For any $\varepsilon$ satisfying 
    $ 0 < \varepsilon \le \min\left\{\frac{1}{e}, 16C_H \Big(\frac{\tau}{4}\Big)^\alpha \right\},$
    there exists a Transformer network 
    \begin{equation*}
        \hat{g}_T \in \mathcal{T}\Big(L, \bar{d}+4, P, \{H^\ell\}_{\ell=1}^L, \{d_k^\ell\}_{\ell=1}^L, \{d_v^\ell\}_{\ell=1}^L, \{d_{\text{ff}}^\ell\}_{\ell=1}^L, M_{max}\Big)
    \end{equation*}
    satisfying 
    \begin{equation*}
        \sup_{\bx \in \cM} |\hat{g}_T(\bx) - g(\bx)| \le \varepsilon
    \end{equation*}
    with the following parameters:
    \begin{itemize}
        \item $L=2$, embedding dimension $D= \bar{d}+4$, and $P \le C_P \varepsilon^{-d/\alpha}$ where $C_P = C_{\cM} (16C_H)^{d/\alpha}$.
        \item For $\ell=1$: $H^1 = P+2$, $d_k^1 = d_v^1 = 2$, and $d_{\text{ff}}^1 = 2\bar{d}+8$.
        \item For $\ell=2$: $H^2 = 1$, $d_k^2 = d_v^2 = 1$, and $d_{\text{ff}}^2 = 2\bar{d}+8$.
    \end{itemize}
    The parameter magnitude $M_{max}$ and total dense parameters $\cN_{total}$ satisfy
    \begin{align*}
        M_{max} \le \widetilde{C}_{mag} \varepsilon^{-2d/\alpha} \log \frac{1}{\varepsilon}, \qquad
        \cN_{total} \le C_N \varepsilon^{-d/\alpha},
    \end{align*}
    where 
    $C_N= 10(\bar{d}+4)C_P + 9\bar{d}^2 + 95\bar{d} + 236$, 
    $\widetilde{C}_{mag} = \max\left\{ \frac{1}{1 - \frac{1}{eB}}, C_{mag}, 2 \left(1 + \frac{1}{2eB}\right) C_B \right\},$ 
    with $C_B = \max\{3\bar{d} C_M, B, 1\}$, $C_{mag} = \frac{C_P^2 C_{log}}{8}$, $C_{log} = \left| \log(2 C_P B (1 + 6 \bar{d} C_M)) \right| + \frac{d+2+\alpha}{\alpha} + 2$, $C_M = \frac{(16C_H)^{2/\alpha}}{3} \left(\log( 4B C_P ) + \frac{d+\alpha}{\alpha} \right)$, $C_{\cM}= 3^d \text{Vol}(\cM)d^{d/2}$, and $\text{Vol}(\cM)$ is the volume of $\cM$.
\end{theorem}

The rigorous establishment of this approximation guarantee relies on several key geometric tools detailed in  Appendix \ref{sec:appendixA2}. Specifically, we leverage the local metric equivalence between the geodesic and Euclidean distances governed by the reach $\tau$, alongside the intrinsic covering number bounds of manifolds (see Lemma \ref{lemma_manifold_geometry}). These geometric properties theoretically validate the extension of our Softmax POU scheme to the manifold setting (see Lemma \ref{lemma_POU_approx_manifold}). The detailed proof of Theorem \ref{thm_manifold_approx} is deferred to Appendix \ref{sec:appendixB2}. The architectural construction follows a similar procedure to the proof of Theorem \ref{thm_holder_approx}, but adapts the scaling factor and sequence length to the 
covering properties of $\cM$.
Based on this approximation theorem, we further establish the statistical convergence rate of the Transformer estimator $f_{\mathcal{S}}$ for target functions defined on manifolds.

\begin{theorem} \label{thm_gen_manifold}
    Suppose that the regression function $f_\rho: \cM \to \RR$ is an $\alpha$-H\"older continuous function with constant $C_H$ and $\|f_\rho\|_\infty \le B$, where $\cM \subseteq [0,1]^{\bar{d}}$ is a compact and connected $d$-dimensional Riemannian manifold with reach $\tau > 0$. Consider the truncated empirical target function $\pi_B f_{\mathcal{S}}$ over the hypothesis space $\mathcal{T}$ defined in Theorem \ref{thm_manifold_approx}, for $0 < \varepsilon \le \min\left\{\frac{1}{e}, 16C_H \Big(\frac{\tau}{4}\Big)^\alpha \right\}$, we have
    \begin{equation*}
        \mathbb{E}\left[\left\|\pi_B f_{\mathcal{S}}-f_\rho\right\|_\rho^2\right] \le \widetilde C_4 \max\left\{ \varepsilon^2, \frac{\varepsilon^{-d/\alpha} \log \frac{1}{\varepsilon}}{n} \right\}.
    \end{equation*}
    Moreover, by choosing $\varepsilon = n^{-\frac{\alpha}{2\alpha+d}}$, we obtain
    \begin{equation*}
        \mathbb{E}\left[\left\|\pi_B f_{\mathcal{S}}-f_\rho\right\|_\rho^2\right] \le \widetilde{C}_4 n^{-\frac{2\alpha}{2\alpha+d}} \log n,
    \end{equation*}
    where $\widetilde C_4$ is a positive constant depending on $\bar{d}, d, \alpha, B, C_H$, and $\text{Vol}(\cM)$ specified in \eqref{C42}.
\end{theorem}

The detailed proof of Theorem \ref{thm_gen_manifold} is deferred to Appendix \ref{sec:appendixC}. Similar to the Euclidean setting, Theorem \ref{thm_gen_manifold} is derived by utilizing the bias-variance tradeoff and the covering number bound of the Transformer hypothesis space. Crucially, notice that the the exponent of the final learning rate depends exclusively on the intrinsic dimension $d$ rather than the ambient dimension $\bar{d}$. This demonstrates that a shallow, wide, and dense Transformer hypothesis class can effectively avoid the curse of dimensionality when learning from data intrinsically supported on low-dimensional manifolds.

\section{Numerical experiment}\label{sec:numerical}

We examine whether Transformer training produces spatially localized attention on a compact manifold. We consider noiseless regression $\bg:\mathcal X\to\mathbb R^8$ on the unit circle with Cartesian inputs $\bx(u)=(\cos(2\pi u-\pi),\sin(2\pi u-\pi))^\top$, $u\in[0,1)$. The eight target functions $g_1,\ldots,g_8$ are generated independently as continuous, periodic, piecewise-linear functions of $u$, each with eight segments (Appendix~\ref{app:num-protocol}, \eqref{eq:num-data}). We jointly fit the eight targets using $1024$ inputs sampled uniformly from the circle. Evaluation uses $4096$ equally spaced circle points.
The Transformer we use has $128$ tokens of dimension $16$ (Appendix~\ref{app:num-protocol}, \eqref{eq:num-encoder}). An affine map inserts $\bx$ into token $0$; the other $127$ tokens have learned embeddings with no prescribed spatial locations. The first block combines two-head self-attention and a ReLU FFN with residual connections, producing $\bZ^{(1)}(\bx)=[\bz_0(\bx),\ldots,\bz_{127}(\bx)]$. To isolate attention-weighted aggregation, the second block uses one-head attention from query token $0$ to sources $\mathcal S=\{1,\ldots,127\}$, with an identity FFN and no residual connection:
\begin{equation}\label{eq:num-readout}
\begin{split}
 a_j(\bx)&=\frac{\exp\{(\bQ_2\bz_0(\bx))^{\mathsf T}\bK_2\bz_j(\bx)\}}
 {\sum_{k\in\mathcal S}\exp\{(\bQ_2\bz_0(\bx))^{\mathsf T}\bK_2\bz_k(\bx)\}},\\
 \bbf(\bx)&=\sum_{j\in\mathcal S}a_j(\bx)\be_j(\bx),\qquad
 \be_j(\bx)=\bW_V^{\rm eff}\bz_j(\bx)+\bb_V\in\mathbb R^8.
\end{split}
\end{equation}
We minimize the squared loss \eqref{eq:num-loss} using $14000$ Adam updates (initial learning rate $0.003$, batch size $256$), followed by up to $400$ full-batch L-BFGS iterations. Learning-rate scheduling and checkpoint selection are specified in Appendix~\ref{app:num-protocol}. The network achieves normalized mean squared error (NMSE; \eqref{eq:num-nmse}) of $9.07\times10^{-4}$.

\begin{figure}[!ht]
\centering
\includegraphics[width=\textwidth]{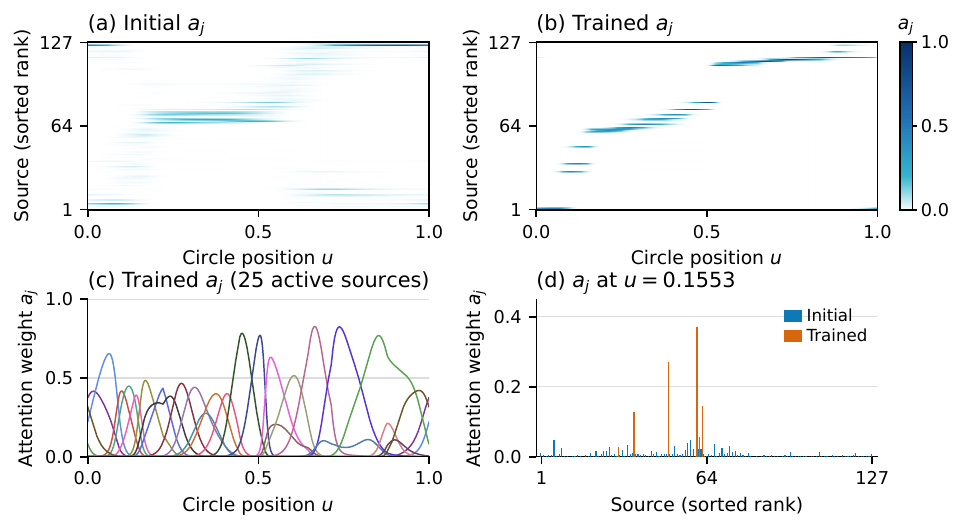}
\caption{Second-layer softmax attention weights $a_j(\bx(u))$. Panels (a,b,d) share the source order sorted by trained calibration peaks. Panel (c) shows all $25$ active sources whose mean trained attention weight over calibration inputs meets the threshold in \eqref{eq:num-peaks}. Selection details: Appendix~\ref{app:num-metrics}.}
\label{fig:numerical}
\end{figure}

Figure~\ref{fig:numerical}(a,b) illustrates how $a_j(\bx(u))$ evolve during training. Initially, attention responses are relatively diffuse across the circle. After training, different sources develop more spatially concentrated responses, becoming prominent in different local regions. Panel~(c) further reveals overlapping localized bumps from the $25$ active sources, many of which exhibit Gaussian-like weights with rapid decay away from their peaks. These patterns resemble the Softmax POU weights in \eqref{eq_softmax_weights}, where localized weights emphasize nearby regions while maintaining smooth transitions through their overlap. Panel~(d) illustrates this input-dependent weighting at $u=0.1553$, where the trained attention concentrates most of its weights on a few sources, in contrast to the more diffuse initial distribution.  Together, these observations suggest that spatially localized, Gaussian-like attention weights can emerge through training without prescribed centers or locality penalties, exhibiting a local-to-global weighting behavior qualitatively aligned with our Softmax POU construction.

\section{Related work and discussion} \label{sec:related_work}

\begin{table}[t]
\centering
\caption{Comparison of the cited Transformer constructions and our work. The reported depth includes the hidden depth in point-wise FFN layers.}
\label{tab:comparison}
\begin{tabular}{cccc}
\toprule
\textbf{References} & \textbf{Softmax Attention} & \textbf{Sinusoidal PE} & \textbf{Depth} \\
\midrule
Gurevych et al. \cite{gurevych2022rate} & \bad & \bad & $\mathcal{O}(1)$ \\
Takakura and Suzuki \cite{takakura2023approximation} & \good & \good & $\mathcal{O}(\text{poly}(\log(1/\varepsilon)))$ \\
Havrilla and Liao  \cite{havrilla2024understanding} & \bad & \good & $\mathcal{O}(\log(1/\varepsilon))$ \\
Shi et al.  \cite{shi2025approximation} & \good & \bad & $\mathcal{O}(\log(1/\varepsilon))$ \\
Jiao et al. \cite{jiao2026transformers} & \good & \bad & $\mathcal{O}(\log(1/\varepsilon))$ \\
\midrule
\textbf{Ours (This work)} & \good & \good & $2$ \\
\bottomrule
\end{tabular}
\end{table}

\textbf{Learning Theory of Transformers.} Recent studies have investigated the quantitative approximation and estimation capabilities of Transformers \cite{gurevych2022rate, takakura2023approximation, jiang2024approximation, havrilla2024understanding, shi2025approximation, jiao2026transformers}. Specifically, they show that Transformers can mitigate the curse of dimensionality when target functions exhibit hierarchical compositional structures \cite{gurevych2022rate, shi2025approximation} or anisotropic smoothness in sequence-to-sequence tasks \cite{takakura2023approximation}. Further research has established Jackson-type rates and comparisons with RNNs \cite{jiang2024approximation}, and utilized the Kolmogorov-Arnold theorem for H\"older function approximation \cite{jiao2026transformers}. Some existing constructions replace softmax with hardmax-based or ReLU attention \cite{gurevych2022rate, havrilla2024understanding}, while others omit sinusoidal PE~\cite{shi2025approximation, jiao2026transformers}. Even when softmax attention is retained, existing approximation constructions still rely primarily on point-wise ReLU FFN layers \cite{takakura2023approximation, shi2025approximation}. In contrast, our framework leverages native softmax attention to implement a Softmax POU for spatial localization and the aggregation of local approximations, while preserving sinusoidal positional encodings. Table~\ref{tab:comparison} summarizes these architectural differences.

\textbf{Learning Theory on Manifolds.} Neural-network approximation and statistical learning have been studied on manifolds and more general low-dimensional sets. Chen et al.~\cite{chen2019efficient, chen2022nonparametric} construct sparse ReLU approximations of smooth manifold functions and establish least-squares regression guarantees, with intrinsic dimension governing the leading rate exponents. Nakada and Imaizumi~\cite{nakada2020adaptive} obtain constant-depth ReLU approximation and regression bounds for H\"older functions with covariates supported on sets of low Minkowski dimension. Cloninger and Klock~\cite{cloninger2021deep} analyze ReLU approximation of functions composed with manifold projections or distances to low-complexity sets, allowing full-dimensional domains. Liu et al.~\cite{liu2021besov} establish Besov-function approximation on manifolds using convolutional residual networks and derive excess-logistic-risk bounds for binary classification. Dahal et al.~\cite{dahal2022deep} prove Wasserstein-1 approximation and estimation guarantees for ReLU generators of manifold-supported distributions with densities bounded above and below. Liu et al.~\cite{liu2024deep} study chart autoencoders trained on paired noisy and clean samples, proving reconstruction and denoising guarantees under bounded normal noise. Havrilla and Liao~\cite{havrilla2024understanding} derive approximation and regression bounds for Transformers with ReLU attention, relating model-size and sample-size scaling exponents to intrinsic dimension and H\"older regularity. Nagler and Langer~\cite{nagler2026optimal} establish dense ReLU approximation on sets of low Minkowski dimension with flexible width--depth tradeoffs, and derive regression rates adapted to intrinsic dimension, smoothness, and compositional structure.

\section{Conclusion} \label{sec:conclusion}
In this paper, we establish a novel local-to-global constructive approximation framework for Transformers based on a globally supported Softmax POU. Leveraging this framework, we provide comprehensive approximation and generalization guarantees for $\alpha$-H\"older continuous functions on both compact Euclidean domains and Riemannian manifolds, where we restrict $\alpha \in (0,1]$. Specifically, we prove that shallow, wide, and dense Transformers with softmax activation and sinusoidal positional encodings can attain near minimax-optimal convergence rates and overcome the curse of dimensionality. A compelling future direction is to explore the potential of our Softmax POU framework for analyzing the mechanisms of in-context learning \cite{brown2020language,dong2024survey,shen2026understanding}. While classic studies often interpret this phenomenon by viewing Transformers as implicit statisticians \cite{bai2023transformers} or as executing internal learning algorithms \cite{garg2022can}, our constructive scheme offers a potentially distinct perspective to further investigate how Transformer architectures achieve adaptive approximations during inference.


\clearpage
\begin{ack}
Zhongjie Shi and Wenjing Liao acknowledge support from the National Science Foundation under grant NSF DMS 2145167 and the U.S. Department of Energy under grant DOE SC0024348. AI tools were used to create Figure~\ref{fig:softpou}, implement, run, and analyze the numerical experiments, generate the experimental plots (Figures~\ref{fig:numerical}--\ref{fig:numerical-appendix}), and polish the language of the manuscript.

\end{ack}

\bibliographystyle{plain}
\bibliography{ref}

\clearpage
\appendix
\section*{Appendix}
This appendix provides the complete mathematical proofs for our theoretical results. It is organized as follows: Section \ref{sec:appendixA} proves the preliminary lemmas for the Softmax POU scheme and Transformer constructions; Section \ref{sec:appendixB} establishes the main approximation guarantees for both Euclidean domains and Riemannian manifolds; Section \ref{sec:appendixC} details the covering number bounds and the derivation of the generalization error rates; and Section~\ref{app:numerical} provides the numerical protocols and additional diagnostics.

\section{Proof of preliminary results} \label{sec:appendixA}
This section establishes the detailed proofs for the foundational mathematical tools used in our theoretical framework. We first prove the Softmax POU approximation schemes and then provide the constructive lemmas for individual Transformer layers.

\subsection{Softmax POU scheme for Euclidean domains in Section \ref{sec:euclidean}}
\label{sec:appendixA1}
In this subsection, we present the proof of Lemma \ref{lemma_POU_approx}. This establishes the core mechanism of using a Softmax POU to achieve local-to-global approximation on compact Euclidean spaces.

\begin{proof}[Proof of Lemma~\ref{lemma_POU_approx}]
    To construct the approximation, we first explicitly establish the finite covering and bound the required number of centers $C_g$. For any target accuracy $\varepsilon \in (0, 1/e]$, we set the covering radius as $r_g = (\varepsilon/(4C_H))^{1/\alpha}$. We partition the domain $[0,1]^d$ into a uniform grid of hypercubes. By strategically setting the side length of each hypercube to $\delta = 2r_g/\sqrt{d}$, the distance from its center to any vertex (i.e., the circumscribed radius) is exactly $\frac{1}{2}\sqrt{d\delta^2} = r_g$. This guarantees that a closed Euclidean ball $\bar{\mathcal{B}}(\bc_k, r_g)$ centered at the midpoint $\bc_k$ of each hypercube completely covers it, yielding a valid finite cover of $[0,1]^d$. 
    
    The number of grid points required along each dimension is exactly $\lceil 1/\delta \rceil$, resulting in a total of $C_g = \lceil 1/\delta \rceil^d$ centers. By the inequality that $\lceil x \rceil \le 2x$ when $x \ge 1$, we have
    \begin{align*}
        C_g &= \left\lceil \frac{1}{\delta} \right\rceil^d = \left\lceil \frac{\sqrt{d}}{2r_g} \right\rceil^d 
        \le \left( 2 \cdot \frac{\sqrt{d}}{2r_g} \right)^d = \left( \frac{\sqrt{d}}{r_g} \right)^d 
        = \left( \sqrt{d} (4C_H)^{1/\alpha} \right)^d \varepsilon^{-\frac{d}{\alpha}}.
    \end{align*}
    
    With the centers $\{\bc_i\}_{i=1}^{C_g}$ and the covering established, for any $\bx \in [0,1]^d$, denote $i^*(\bx) = \argmin_{i \in [C_g]} \|\bx - \bc_i\|_2$. Since the closed balls cover the domain, we are guaranteed that $\|\bx - \bc_{i^*}\|_2 \le r_g$. 
    
    Next, we bound the approximation error. Notice that $\beta_i(\bx) = \frac{\exp\left(M_g (r_g^2 - \|\bx - \bc_i\|_2^2)\right)}{\sum_{k=1}^{C_g} \exp\left(M_g (r_g^2 - \|\bx - \bc_k\|_2^2)\right)}$, using the exact POU property that $\sum_{i=1}^{C_g} \beta_i(\bx) = 1$, we decompose the error into near and far components at a radius of $2r_g$, that is
    \begin{align*}
        |\hat g(\bx) - g(\bx)| &= \left|\sum_{i=1}^{C_g} \beta_i(\bx) \big(g(\bc_i) - g(\bx)\big)\right| \\
        &\le \underbrace{\sum_{i: \|\bx - \bc_i\|_2 \le 2r_g} \beta_i(\bx) |g(\bx) - g(\bc_i)|}_{(I)} + \underbrace{\sum_{i: \|\bx - \bc_i\|_2 > 2r_g} \beta_i(\bx) |g(\bx) - g(\bc_i)|}_{(II)}.
    \end{align*}
    
    For the near-center terms $(I)$, the $\alpha$-H\"older continuity yields $|g(\bx) - g(\bc_i)| \le C_H (2r_g)^\alpha \le 2 C_H r_g^\alpha$ since $2^\alpha \le 2$ for $\alpha \in (0,1]$. By substituting our explicit choice of $r_g$, we obtain
    $$(I) \le 2 C_H r_g^\alpha \sum_{i} \beta_i(\bx) \le 2 C_H r_g^\alpha = 2 C_H \left( \frac{\varepsilon}{4C_H} \right) = \frac{\varepsilon}{2}.$$ 
    
    For the far-center tail terms $(II)$, the condition $\|\bx - \bc_i\|_2 > 2r_g$ implies $r_g^2 - \|\bx - \bc_i\|_2^2 < -3r_g^2$. Since $\|\bx - \bc_{i^*}\|_2 \le r_g$, the denominator of $\beta_i(\bx)$ is bounded below
    \begin{equation*}
        \sum_{k=1}^{C_g} \exp\left(M_g(r_g^2 - \|\bx - \bc_k\|_2^2)\right) \ge \exp\left(M_g(r_g^2 - \|\bx - \bc_{i^*}\|_2^2)\right) \ge \exp(0) = 1.
    \end{equation*}
    Applying the global bound $|g(\bx) - g(\bc_i)| \le 2\|g\|_\infty$, the term $(II)$ is bounded by
    \begin{align*}
        (II) &= \sum_{i: \|\bx - \bc_i\|_2 > 2r_g} \frac{\exp\left(M_g(r_g^2 - \|\bx - \bc_i\|_2^2)\right)}{\sum_{k=1}^{C_g} \exp\left(M_g(r_g^2 - \|\bx - \bc_k\|_2^2)\right)} |g(\bx) - g(\bc_i)| \\
        &\le \sum_{i: \|\bx - \bc_i\|_2 > 2r_g} \exp(-3M_g r_g^2) \cdot 2\|g\|_\infty \\
        &\le 2 C_g \|g\|_\infty \exp(-3 M_g r_g^2).
    \end{align*}
    
    To guarantee $(II) \le \varepsilon/2$, the scaling parameter must satisfy $M_g \ge (3r_g^2)^{-1}\log(4C_g\|g\|_\infty/\varepsilon)$. Substituting the explicit formulas for $r_g$ and $C_g$, we require
    \begin{align*}
        M_g &\ge \frac{(4C_H)^{2/\alpha}}{3} \varepsilon^{-\frac{2}{\alpha}} \log\left( 4\|g\|_\infty \left( \sqrt{d} (4C_H)^{1/\alpha}\right)^d \varepsilon^{-\frac{d+\alpha}{\alpha}} \right) \\
        &= \frac{(4C_H)^{2/\alpha}}{3} \varepsilon^{-\frac{2}{\alpha}} \left( \log\left( 4\|g\|_\infty \left( \sqrt{d} (4C_H)^{1/\alpha}\right)^d \right) + \frac{d+\alpha}{\alpha} \log \frac{1}{\varepsilon} \right).
    \end{align*}
    Assuming $\varepsilon \le 1/e$ so that $\log(1/\varepsilon) \ge 1$, we can upper bound the bracketed sum by factoring out $\log(1/\varepsilon)$. Defining $M_g$ as
    \begin{equation*}
        M_g = \frac{(4C_H)^{2/\alpha}}{3} \left( \log\left( 4\|g\|_\infty \left( \sqrt{d} (4C_H)^{1/\alpha}\right)^d \right) + \frac{d+\alpha}{\alpha} \right) \varepsilon^{-\frac{2}{\alpha}} \log \frac{1}{\varepsilon}
    \end{equation*}
    satisfies this condition, ensuring $(II) \le \varepsilon/2$. 
    
    Combining $(I)$ and $(II)$ achieves the target uniform error $$\sup_{\bx \in [0,1]^d}| \hat g(\bx) - g(\bx) | \le \frac{\varepsilon}{2} + \frac{\varepsilon}{2} = \varepsilon.$$ 
    
    Finally, notice that the quadratic term $\|\bx\|_2^2$ exactly cancels out in the Softmax expansion, we have
    \begin{align*}
    	\beta_i(\bx) &= \frac{\exp\left(M_g(r_g^2 - \|\bx\|_2^2)\right) \exp\left(2M_g \inner{\bx}{\bc_i} - M_g \|\bc_i\|_2^2\right)}{\sum_{k=1}^{C_g} \exp\left(M_g(r_g^2 - \|\bx\|_2^2)\right) \exp\left(2M_g \inner{\bx}{\bc_k} - M_g \|\bc_k\|_2^2\right)} \nonumber \\
    	&= \frac{\exp\left(2M_g \inner{\bx}{\bc_i} - M_g \|\bc_i\|_2^2\right)}{\sum_{k=1}^{C_g} \exp\left(2M_g \inner{\bx}{\bc_k} - M_g \|\bc_k\|_2^2\right)},
    \end{align*}
    thus we complete the proof.
\end{proof}

\subsection{Softmax POU scheme for manifolds in Section \ref{sec:mtor}} \label{sec:appendixA2}

Here, we extend our Softmax POU approximation scheme to compact Riemannian manifolds (See Lemma \ref{lemma_POU_approx_manifold}). We leverage the intrinsic geometric properties and local metric equivalence to properly bound the approximation errors.
We first introduce some definitions and notations about Riemannian manifolds. Let $d(\xb,\cM) := \inf_{\vb\in\cM}\|\xb-\vb\|_2$ denote the Euclidean distance from a point $\xb \in \RR^{\bar d}$ to a manifold $\cM \subset \RR^{\bar d}$. To characterize the curvature and the folding of $\cM$ within the ambient space, we introduce the concept of reach.

\begin{definition}[Reach \citep{federer1959curvature,niyogi2008finding}] \label{def.reach}
    The reach of $\cM$ is defined as
    \begin{equation*}
        \tau = \inf_{\vb\in\cM} \inf_{\xb\in G} \|\xb-\vb\|_2,
    \end{equation*}
    where $G = \left\{\xb\in \RR^{\bar d}: \exists \mbox{ distinct } \pb,\qb\in \cM \mbox{ such that } d(\xb,\cM)=\|\xb-\pb\|_2=\|\xb-\qb\|_2\right\}$ is the medial axis of $\cM$.
\end{definition}

We equip the compact and connected Riemannian manifold $\cM$ with the intrinsic geodesic metric $d_{\cM}: \cM \times \cM \to \RR_{\ge 0}$, defined as
\begin{equation*}
    d_{\cM}(\bx, \by) := \inf \left\{ \int_0^1 \|\dot{\gamma}(t)\|_2 \, dt \;\middle|\; \gamma \in C^1([0,1], \cM), \gamma(0)=\bx, \gamma(1)=\by \right\}.
\end{equation*}

Based on this metric space $(\cM, d_{\cM})$, we define the H\"older continuous functions on the manifolds.

\begin{definition}[H\"older Continuous Functions on Manifolds] \label{def_holder}
    Let $\alpha \in (0, 1]$, $(\cM, d_{\cM})$ be a compact and connected $d$-dimensional Riemannian manifold. A function $f: \cM \to \RR$ is $\alpha$-H\"older continuous if there exists a constant $C_H \ge 0$ such that 
    \begin{equation*}
        |f(\bx) - f(\by)| \le C_H \cdot [d_{\cM}(\bx, \by)]^\alpha, \quad \forall \bx, \by \in \cM.
    \end{equation*}
\end{definition}

Recall that the $\epsilon$-covering number $\cN(X, \epsilon, d)$ is the minimum number of open balls of radius $\epsilon$ needed to cover a metric space $(X, d)$. Since the Euclidean chordal distance is naturally upper-bounded by the intrinsic geodesic distance (i.e., $\|\bx - \by\|_2 \le d_{\cM}(\bx, \by)$), any valid $\epsilon$-covering under $d_{\cM}$ is strictly an $\epsilon$-covering under $\|\cdot\|_2$. This topological inclusion directly yields $\cN(\cM, \epsilon, \|\cdot\|_2) \le \cN(\cM, \epsilon, d_{\cM})$.
The following lemma establishes the bounds for these covering numbers and the local metric equivalence.

\begin{lemma} \label{lemma_manifold_geometry}
    Let $\cM \subseteq [0,1]^{\bar d}$ be a $d$-dimensional compact and connected Riemannian manifold with reach $\tau > 0$. 
    \begin{enumerate}
        \item (Metric Equivalence \cite[Lemma 3]{genovese2012minimax}) For any $\bx, \by \in \cM$ such that $\|\bx - \by\|_2 \le \tau/2$, the intrinsic geodesic distance satisfies
        \begin{equation} \label{eq_geodesic_bound}
            d_{\cM}(\bx, \by) \le \tau \left( 1 - \sqrt{1 - \frac{2\|\bx - \by\|_2}{\tau}} \right) \le 2 \|\bx - \by\|_2.
        \end{equation}
        \item (Covering Number Bound \cite{niyogi2008finding}) For any $0 < \epsilon \le \tau/2$, the covering numbers of $\cM$ are bounded by
        \begin{equation} \label{eq_covering_bound}
            \cN(\cM, \epsilon, \|\cdot\|_2) \le \cN(\cM, \epsilon, d_{\cM}) \le C_{\cM} \epsilon^{-d},
        \end{equation}
        where $C_{\cM}= 3^d \text{Vol}(\cM)d^{d/2}$, and $\text{Vol}(\cM)$ is the volume of $\cM$.
    \end{enumerate}
\end{lemma}

With the local metric equivalence and intrinsic covering number bound of manifolds established, we proceed to the Softmax POU approximation scheme on Riemannian manifolds. The proof controls the near-center error using the intrinsic H\"older continuity, and bounds the far-field error through the exponential decay of the ambient Softmax weights.

\begin{lemma}[Softmax POU Approximation on Manifolds] \label{lemma_POU_approx_manifold}
    Let $\cM \subseteq [0,1]^{\bar d}$ be a compact and connected $d$-dimensional Riemannian manifold with reach $\tau > 0$, and let $g: \cM \to \RR$ be an $\alpha$-H\"older continuous function with constant $C_H$. For any target accuracy $\varepsilon$ satisfying 
    $0 < \varepsilon \le \min\left\{\frac{1}{e}, 8C_H \Big(\frac{\tau}{4}\Big)^\alpha \right\},$
    let the covering radius be $r_g = (\varepsilon/(8C_H))^{1/\alpha}$. There exists a finite set of $C_g$ centers $\{\bc_i\}_{i=1}^{C_g} \subset \cM$ forming an open $r_g$-covering of the manifold (i.e., $\cM \subseteq \bigcup_{i=1}^{C_g} \mathcal{B}(\bc_i, r_g)$) with $C_g \leq C_{\cM}(8C_H)^{d/\alpha} \varepsilon^{-d/\alpha}$. By setting the scaling parameter
    \begin{equation*}
        M_g = C_M \varepsilon^{-2/\alpha} \log \frac{1}{\varepsilon}, \quad \text{where } C_M = \frac{(8C_H)^{2/\alpha}}{3} \left( \log\left(4 \|g\|_\infty C_{\cM} (8C_H)^{d/\alpha}\right) + \frac{d+\alpha}{\alpha} \right),
    \end{equation*}
    the ambient Softmax POU approximation
    $$\hat{g}(\bx) = \sum_{i=1}^{C_g} \beta_i(\bx) g(\bc_i),$$
    with
    \begin{align}
        \beta_i(\bx) = \frac{\exp\left(M_g (r_g^2 - \|\bx - \bc_i\|_2^2)\right)}{\sum_{k=1}^{C_g} \exp\left(M_g (r_g^2 - \|\bx - \bc_k\|_2^2)\right)}= \frac{\exp\left(2M_g \inner{\bx}{\bc_i} - M_g \|\bc_i\|_2^2\right)}{\sum_{k=1}^{C_g} \exp\left(2M_g \inner{\bx}{\bc_k} - M_g \|\bc_k\|_2^2\right)},
    \end{align}
    achieves the uniform approximation bound
    $$\sup_{\bx \in \cM} |\hat{g}(\bx) - g(\bx)| \le \varepsilon.$$
\end{lemma}

\begin{proof}[Proof of Lemma~\ref{lemma_POU_approx_manifold}]
    To construct the approximation, we first establish a finite covering of the data manifold $\cM$ by open Euclidean balls $\{\mathcal{B}(\bc_i, r_g)\}_{i=1}^{C_g}$ of radius $r_g > 0$, with centers $\bc_i \in \cM$. By the covering number bound in Lemma \ref{lemma_manifold_geometry}, such an open covering exists with the number of centers bounded by $C_g \le C_{\cM} r_g^{-d}$. 
    For any $\bx \in \cM$, let $i^*(\bx) = \argmin_{i \in [C_g]} \|\bx - \bc_i\|_2$. 
    Since these open balls cover $\cM$, we are guaranteed that $\|\bx - \bc_{i^*}\|_2 < r_g$. Notice that $\beta_i(\bx) = \frac{\exp\left(M_g (r_g^2 - \|\bx - \bc_i\|_2^2)\right)}{\sum_{k=1}^{C_g} \exp\left(M_g (r_g^2 - \|\bx - \bc_k\|_2^2)\right)}$, using the property that $\sum_{i=1}^{C_g} \beta_i(\bx) = 1$, by decomposing the approximation error into near and far components at a Euclidean radius of $2r_g$, we have
    \begin{align*}
        |\hat g(\bx) - g(\bx)| &= \left|\sum_{i=1}^{C_g} \beta_i(\bx) \big(g(\bc_i) - g(\bx)\big)\right| \\
        &\le \underbrace{\sum_{i: \|\bx - \bc_i\|_2 \le 2r_g} \beta_i(\bx) |g(\bx) - g(\bc_i)|}_{(I)} + \underbrace{\sum_{i: \|\bx - \bc_i\|_2 > 2r_g} \beta_i(\bx) |g(\bx) - g(\bc_i)|}_{(II)}.
    \end{align*}
    
    For the near-center terms $(I)$, we can bound the error using the intrinsic H\"older continuity. For sufficiently small $\varepsilon$ such that $2r_g \le \tau/2$, by Lemma \ref{lemma_manifold_geometry}, we have $d_{\cM}(\bx, \bc_i) \le 2\|\bx - \bc_i\|_2 \le 4r_g$. Consequently, the intrinsic $\alpha$-H\"older continuity yields 
    $$|g(\bx) - g(\bc_i)| \le C_H (d_{\cM}(\bx, \bc_i))^\alpha \le C_H (4r_g)^\alpha \le 4 C_H r_g^\alpha,$$ since $4^\alpha \le 4$ for $\alpha \in (0,1]$. Thus
    $$(I) \le 4 C_H r_g^\alpha \sum_{i} \beta_i(\bx) \le 4 C_H r_g^\alpha.$$
    
    For the far-center tail terms $(II)$, the condition $\|\bx - \bc_i\|_2 > 2r_g$ implies $r_g^2 - \|\bx - \bc_i\|_2^2 < -3r_g^2$. Since $\|\bx - \bc_{i^*}\|_2 \le r_g$, the denominator of the ambient Softmax POU $\beta_i(\bx)$ is bounded below
    \begin{equation*}
        \sum_{k=1}^{C_g} \exp\left(M_g(r_g^2 - \|\bx - \bc_k\|_2^2)\right) \ge \exp\left(M_g(r_g^2 - \|\bx - \bc_{i^*}\|_2^2)\right) \ge \exp(0) = 1.
    \end{equation*}
    Applying the global bound $|g(\bx) - g(\bc_i)| \le 2\|g\|_\infty$, the term $(II)$ is bounded by
    \begin{align*}
        (II) &= \sum_{i: \|\bx - \bc_i\|_2 > 2r_g} \frac{\exp\left(M_g(r_g^2 - \|\bx - \bc_i\|_2^2)\right)}{\sum_{k=1}^{C_g} \exp\left(M_g(r_g^2 - \|\bx - \bc_k\|_2^2)\right)} |g(\bx) - g(\bc_i)| \\
        &\le \sum_{i: \|\bx - \bc_i\|_2 > 2r_g} \exp(-3M_g r_g^2) \cdot 2\|g\|_\infty \\
        &\le 2 C_g \|g\|_\infty \exp(-3 M_g r_g^2).
    \end{align*}
    
    Combining $(I)$ and $(II)$ yields
    \begin{equation} \label{eq_POU_error_proof_manifold}
        \sup_{\bx \in \cM} | \hat g(\bx) - g(\bx) | \le 4 C_H r_g^\alpha + 2 C_g \|g\|_\infty \exp(-3 M_g r_g^2).
    \end{equation}
    
    Setting $r_g = (\varepsilon/(8C_H))^{1/\alpha}$ yields $(I) \le \varepsilon/2$. By Lemma \ref{lemma_manifold_geometry}, the number of Euclidean balls of radius $r_g$ required to cover $\cM$ is bounded by $C_g \le C_{\cM} r_g^{-d}$, where $C_{\cM} = 3^d \text{Vol}(\cM) d^{d/2}$. Substituting our choice of $r_g$, we obtain
    \begin{align*}
        C_g &\le C_{\cM} \left( \left(\frac{\varepsilon}{8C_H}\right)^{1/\alpha} \right)^{-d} = C_{\cM} (8C_H)^{d/\alpha} \varepsilon^{-\frac{d}{\alpha}}.
    \end{align*}
    
    Finally, to guarantee $(II) \le \varepsilon/2$, the scaling parameter must satisfy $M_g \ge (3r_g^2)^{-1}\log(4\|g\|_\infty C_g/\varepsilon)$. Substituting the explicit formulas for $r_g$ and the upper bound for $C_g$, we require
    \begin{align*}
        M_g &\ge \frac{(8C_H)^{2/\alpha}}{3} \varepsilon^{-\frac{2}{\alpha}} \log\left( 4\|g\|_\infty C_{\cM} (8C_H)^{d/\alpha} \varepsilon^{-\frac{d+\alpha}{\alpha}} \right) \\
        &= \frac{(8C_H)^{2/\alpha}}{3} \varepsilon^{-\frac{2}{\alpha}} \left( \log\left( 4\|g\|_\infty C_{\cM} (8C_H)^{d/\alpha} \right) + \frac{d+\alpha}{\alpha} \log \frac{1}{\varepsilon} \right).
    \end{align*}
    Assuming $\varepsilon \le 1/e$ so that $\log(1/\varepsilon) \ge 1$, we can upper bound the bracketed sum by factoring out $\log(1/\varepsilon)$. Defining $M_g$ as
    \begin{equation*}
        M_g = \frac{(8C_H)^{2/\alpha}}{3} \left( \log\left( 4\|g\|_\infty C_{\cM} (8C_H)^{d/\alpha} \right) + \frac{d+\alpha}{\alpha} \right) \varepsilon^{-\frac{2}{\alpha}} \log \frac{1}{\varepsilon}
    \end{equation*}
    satisfies this condition, ensuring $(II) \le \varepsilon/2$. This achieves the target uniform error $\sup_{\bx \in \cM}| \hat g(\bx) - g(\bx) | \le \varepsilon$ and completes the proof.
\end{proof}

\subsection{Constructing Transformers for Softmax POU approximation: fundamental lemmas}
\label{sec:appendixA3}

This subsection details the the structural lemmas (Lemmas \ref{lemma_preprocessing}, \ref{lemma_mha_affine}, \ref{lemma_FFN_restore}, and \ref{lemma_mha_pou}) and their proofs. These proofs demonstrate how the native softmax attention mechanism in Transformers can realize our Softmax POU approximation framework. The pre-processing step embeds $\bx$ and an indicator exclusively into the first token, zero-pads the remaining sequence, and appends sinusoidal positional encodings across all $P$ tokens.

\begin{lemma}[Pre-processing Step] \label{lemma_preprocessing}
    Let $\bx \in [0,1]^d$ be the input vector, $P \ge 1$ be the sequence length, and $D = d + 4$ be the embedding dimension. There exists a pre-processing operator $\cP: \RR^d \to \RR^{D \times P}$ with the absolute values of the parameters bounded by $1$, such that its output $\bZ_0 = \cP(\bx)$ is
    \begin{equation*}
        \bZ_0 = \begin{bmatrix}
            \bx & \bm{0} & \cdots & \bm{0} \\
            1 & 0 & \cdots & 0 \\
            0 & 0 & \cdots & 0 \\
            \sin(\theta_1) & \sin(\theta_2) & \cdots & \sin(\theta_P) \\
            \cos(\theta_1) & \cos(\theta_2) & \cdots & \cos(\theta_P)
        \end{bmatrix} \in \RR^{D \times P},
    \end{equation*}
    where $\theta_j = \frac{2\pi j}{P}$ for $j \in [P]$.
\end{lemma}

\begin{proof}[Proof of Lemma~\ref{lemma_preprocessing}]
    Let $\be_1 = [1, 0, \dots, 0]^\top \in \RR^P$ be the indicator vector for the first token. We construct the pre-processing operator as
    \begin{equation*}
        \cP(\bx) = (\bW_E \bx + \bb_E) \be_1^\top + \bP.
    \end{equation*}
    We define the embedding weight matrix $\bW_E \in \RR^{D \times d}$ and the bias vector $\bb_E \in \RR^D$ as
    \begin{equation*}
        \bW_E = \begin{bmatrix} \bI_d \\ \bm{0}_{1 \times d} \\ \bm{0}_{1 \times d} \\ \bm{0}_{2 \times d} \end{bmatrix}, \quad \bb_E = \begin{bmatrix} \bm{0}_{d \times 1} \\ 1 \\ 0 \\ \bm{0}_{2 \times 1} \end{bmatrix},
    \end{equation*}
    and the positional encoding matrix $\bP \in \RR^{D \times P}$ as
    \begin{equation*}
        \bP = \begin{bmatrix} \bm{0}_{d \times P} \\ \bm{0}_{1 \times P} \\ \bm{0}_{1 \times P} \\ \begin{matrix} \sin(\theta_1) & \dots & \sin(\theta_P) \\ \cos(\theta_1) & \dots & \cos(\theta_P) \end{matrix} \end{bmatrix},
    \end{equation*}
    where $\theta_j = \frac{2\pi j}{P}$ for $j \in [P]$. Therefore, we have
    \begin{equation*}
        \bZ_0 = \cP(\bx)= \begin{bmatrix}
            \bx & \bm{0} & \cdots & \bm{0} \\
            1 & 0 & \cdots & 0 \\
            0 & 0 & \cdots & 0 \\
            \sin(\theta_1) & \sin(\theta_2) & \cdots & \sin(\theta_P) \\
            \cos(\theta_1) & \cos(\theta_2) & \cdots & \cos(\theta_P)
        \end{bmatrix} \in \RR^{D \times P},
    \end{equation*}
    thus we complete the proof.
\end{proof}

The first MHA layer extracts two affine transformatioin features, while the exact angular identity $\theta_j$ of the sinusoidal positional encodings remains invariant up to a uniform scalar scaling.
\begin{lemma}[Parallel Affine Transformations via MHA] \label{lemma_mha_affine}
    Let $\bZ_0 \in \RR^{D \times P}$ ($D = d+4$) as defined in Lemma \ref{lemma_preprocessing}. For target affine functions $T_h(\bx) = \bxi_h^\top \bx + b_h$ and $R_h(\bx) = \bzeta_h^\top \bx + c_h$ ($h \in [P]$), let $B_T := \max_h \|T_h\|_\infty$, $B_R := \max_h \|R_h\|_\infty$, and $B_{\cA_1} := \max_h \max \{ \|\bxi_h\|_\infty, \|\bzeta_h\|_\infty, |b_h|, |c_h| \}$. There exists an MHA layer $\cA_1: \RR^{D \times P} \to \RR^{D \times P}$ with $H = P+2$ heads, $d_v= d_k= 2$, and scaling factor $M > 0$, yielding output $\widehat{\bZ}_1= \cA_1(\bZ_0) \in \RR^{D \times P}$ whose $j$-th column is
    \begin{equation*}
        (\widehat{\bZ}_1)_{:, j} = \begin{bmatrix}
            \tilde{T}_j(\bx) \\
            \tilde{R}_j(\bx) \\
            \bm{0}_{d-1} \\
            \tilde{I}_j \\
            \lambda(M) \sin(\theta_j) \\
            \lambda(M) \cos(\theta_j)
        \end{bmatrix}, \quad \forall j \in [P].
    \end{equation*}
    Let $c = 1 - \cos(2\pi/P) > 0$. The output components and parameter magnitude satisfy
    \begin{enumerate}
        \item \textbf{Affine Outputs:} we have the approximation error
        \begin{align*}
            |T_j(\bx) - \tilde{T}_j(\bx)| &\le (P-1) B_T e^{-cM}, \\
            |R_j(\bx) - \tilde{R}_j(\bx)| &\le (P-1) B_R e^{-cM}.
        \end{align*}
        
        \item \textbf{Indicator Output:} $\tilde{I}_1 = \eta(M) := \frac{e^M}{\sum_{l=1}^P e^{M\cos(\theta_l)}}$, and
        \begin{align*}
            0 &\le 1 - \tilde{I}_1 = 1 - \eta(M) \le (P-1)e^{-cM}, \quad &(j = 1) \\
            0 &\le \tilde{I}_j \le e^{-cM}. \quad &(j \neq 1)
        \end{align*}
        
        \item \textbf{PE Output:} the scaling coefficient $\lambda(M)$ satisfies
        \begin{equation*}
            \lambda(M): = \frac{\sum_{k=1}^P e^{M \cos(\theta_k)} \cos(\theta_k)}{\sum_{l=1}^P e^{M \cos(\theta_l)}}, \quad 0 \le 1 - \lambda(M) \le 2(P-1) e^{-c M}.
        \end{equation*}

        \item \textbf{Parameter Bounds:} the maximum parameter magnitude $M_{\cA_1}$ satisfy
        \begin{align*}
            M_{\cA_1} &\le \max \big\{1, M, (1 + (P-1)e^{-cM}) B_{\cA_1} \big\}.
        \end{align*}
    \end{enumerate}
\end{lemma}

\begin{proof}[Proof of Lemma \ref{lemma_mha_affine}]
    We set the uniform dimension for query, key, and value matrices to $d_k = 2$ and $d_v = 2$.
    
    \textbf{Step 1: Affine Heads ($h \in [P]$).} 
    We construct the first $P$ heads in the first encoder block to selectively route the affine transformations. We define the 2D rotation matrix $\bR(\theta)$ as
    \begin{equation*}
        \bR(\theta) = \begin{bmatrix} \cos\theta & \sin\theta \\ -\sin\theta & \cos\theta \end{bmatrix}.
    \end{equation*}
    
    Before defining the attention matrices, we introduce a deterministic scalar $\eta(M)$ representing the exact peak attention weight, that is
    \begin{equation*}
        \eta(M) := \frac{e^M}{\sum_{l=1}^P e^{M\cos(\theta_l)}}.
    \end{equation*}
    To eliminate the influence of the attention matrix while realizing affine transformations, we preemptively rescale the value matrix by $\eta(M)^{-1}$. Adhering to the uniform value dimension $d_v=2$, we stack the weights and biases of both affine transformations. The query, key, and value matrices ($\bQ_1^h, \bK_1^h, \bV_1^h \in \RR^{2 \times D}$) are defined as
    \begin{align*}
        \bQ_1^h &= \begin{bmatrix} 
            \bm{0}_{2 \times d} & \bm{0}_{2 \times 2} & \bR(\theta_1 - \theta_h)
        \end{bmatrix} \implies (\bQ_1^h \bZ_0)_{:,j} = \begin{bmatrix} \sin(\theta_j + \theta_1 - \theta_h) \\ \cos(\theta_j + \theta_1 - \theta_h) \end{bmatrix}, \\
        \bK_1^h &= \begin{bmatrix} 
            \bm{0}_{2 \times d} & \bm{0}_{2 \times 2} & M \bI_2 
        \end{bmatrix} \implies (\bK_1^h \bZ_0)_{:,k} = \begin{bmatrix} M\sin(\theta_k) \\ M\cos(\theta_k) \end{bmatrix}, \\
        \bV_1^h &= \frac{1}{\eta(M)}\begin{bmatrix} 
            \bxi_h^\top & b_h & \bm{0}_{1 \times 3} \\
            \bzeta_h^\top & c_h & \bm{0}_{1 \times 3}
        \end{bmatrix} \\
        & \implies (\bV_1^h \bZ_0)_{:,k} = \frac{1}{\eta(M)}\begin{bmatrix} \bxi_h^\top \bx + b_h \\ \bzeta_h^\top \bx + c_h \end{bmatrix} \delta_{k,1} = \frac{1}{\eta(M)}\begin{bmatrix} T_h(\bx) \\ R_h(\bx) \end{bmatrix} \delta_{k,1}.
    \end{align*}
    
    Following the attention mechanism, the unnormalized score $s_{k,j}^h$ is defined as the inner product between the $k$-th key column and the $j$-th query column, that is,
    \begin{equation*}
        s_{k,j}^h = (\bK_1^h \bZ_0)_{:,k}^\top (\bQ_1^h \bZ_0)_{:,j} = M \cos\big(\theta_k - (\theta_j + \theta_1 - \theta_h)\big).
    \end{equation*}
    Consequently, the $(k, j)$-th entry of the attention matrix $\bA_1^h \in \RR^{P \times P}$ is computed via the column-wise softmax function, we have
    \begin{equation*}
        (\bA_1^h)_{k,j} = \frac{\exp(s_{k,j}^h)}{S_j}, \quad \text{where } S_j = \sum_{l=1}^P \exp(s_{l,j}^h).
    \end{equation*}
    Let $c = 1 - \cos(2\pi/P) > 0$ denote the spectral gap. For any $\theta_a \neq \theta_b$, we have $\cos(\theta_a - \theta_b) \le 1 - c$. We analyze the columns of $\bA_1^h$ (indexed by the query token $j$) by bounding the numerator and the denominator $S_j$ in the following.
    \begin{itemize}
        \item \textbf{Target Column ($j=h$):} The relative phase shift perfectly aligns the attention peak with the source key $k=1$. Substituting $j=h$ yields $s_{k,h}^h = M\cos(\theta_k - \theta_1)$. For the source key ($k=1$), the score reaches its exact maximum $s_{1,h}^h = M\cos(0) = M$. For the denominator $S_h = \sum_{l=1}^P e^{M\cos(\theta_l - \theta_1)}$, subtracting the constant phase $\theta_1$ merely applies a cyclic permutation to the uniform grid. Due to the shift-invariance of the periodic grid, we have $S_h= \sum_{l=1}^P e^{M\cos(\theta_l)}$. Thus, the attention weight for the source key is
        \begin{equation*}
            (\bA_1^h)_{1,h} = \frac{e^{s_{1,h}^h}}{S_h} = \frac{e^M}{\sum_{l=1}^P e^{M\cos(\theta_l)}} = \eta(M),
        \end{equation*}
        therefore, we have
        \begin{equation*}
             (\bV_1^h \bZ_0)_{:,1} (\bA_1^h)_{1,h} = \eta(M) \cdot \frac{1}{\eta(M)} \begin{bmatrix} T_h(\bx) \\ R_h(\bx) \end{bmatrix} = \begin{bmatrix} T_h(\bx) \\ R_h(\bx) \end{bmatrix}.
        \end{equation*}
        Because $(\bV_1^h \bZ_0)_{:,k} = \bm{0}$ for $k \neq 1$, we get
        \begin{equation*}
            (\bV_1^h \bZ_0) (\bA_1^h)_{:,h} = \sum_{k=1}^P (\bA_1^h)_{k,h} (\bV_1^h \bZ_0)_{:,k} = (\bA_1^h)_{1,h} (\bV_1^h \bZ_0)_{:,1} = \begin{bmatrix} T_h(\bx) \\ R_h(\bx) \end{bmatrix}.
        \end{equation*}
        
       \item \textbf{Other Columns ($j \neq h$):} The query perfectly matches a shifted key position $k^*$ satisfying $\theta_{k^*} \equiv \theta_j + \theta_1 - \theta_h$. Since the target differs ($j \neq h$), the attention peak shifts away from the source token ($k^* \neq 1$). The unnormalized attention score at the source key ($k=1$) is
        \begin{equation*}
            s_{1,j}^h = M\cos\big(\theta_1 - (\theta_j + \theta_1 - \theta_h)\big) = M\cos(\theta_h - \theta_j) \le M(1-c).
        \end{equation*}
        We denote the output for these non-target columns as $(\bV_1^h \bZ_0) (\bA_1^h)_{:,j} := [\epsilon_{h,j}, \tilde{\epsilon}_{h,j}]^\top$, then
        \begin{align*}
            \begin{bmatrix} \epsilon_{h,j} \\ \tilde{\epsilon}_{h,j} \end{bmatrix} &= (\bV_1^h \bZ_0) (\bA_1^h)_{:,j} = \sum_{k=1}^P (\bV_1^h \bZ_0)_{:,k} (\bA_1^h)_{k,j} \\
            &= (\bV_1^h \bZ_0)_{:,1} (\bA_1^h)_{1,j} = \frac{(\bA_1^h)_{1,j}}{\eta(M)} \begin{bmatrix} T_h(\bx) \\ R_h(\bx) \end{bmatrix}.
        \end{align*}
        Due to the exact shift-invariance, $S_j = \sum_{l=1}^P e^{M\cos(\theta_l - \theta_j)} = \sum_{l=1}^P e^{M\cos(\theta_l)} \equiv S_h$, thus
        \begin{align*}
            \epsilon_{h,j} &= \frac{e^{s_{1,j}^h} / S_j}{e^M / S_h} T_h(\bx) = e^{s_{1,j}^h - M} T_h(\bx) \implies |\epsilon_{h,j}| \le e^{M(1-c) - M} B_T = e^{-cM} B_T, \\
            \tilde{\epsilon}_{h,j} &= \frac{e^{s_{1,j}^h} / S_j}{e^M / S_h} R_h(\bx) = e^{s_{1,j}^h - M} R_h(\bx) \implies |\tilde{\epsilon}_{h,j}| \le e^{M(1-c) - M} B_R = e^{-cM} B_R.
        \end{align*}
    \end{itemize}

    \textbf{Step 2: Indicator \& Positional Encoding Heads ($h \in \{P+1, P+2\}$).}
    We dedicate Head $P+1$ to preserve the indicator $\delta_{j,1}$ and Head $P+2$ to preserve the positional encodings. The query, key, and value matrices ($\bQ_1^h, \bK_1^h, \bV_1^h \in \RR^{2 \times D}$) are defined as
    \begin{align*}
        \bQ_1^h &= \begin{bmatrix} 
            \bm{0}_{2 \times d} & \bm{0}_{2 \times 2} & \bI_2 
        \end{bmatrix} \implies (\bQ_1^h \bZ_0)_{:,j} = \begin{bmatrix} \sin(\theta_j) \\ \cos(\theta_j) \end{bmatrix}, \quad \forall h \in \{P+1, P+2\}, \\
        \bK_1^h &= \begin{bmatrix} 
            \bm{0}_{2 \times d} & \bm{0}_{2 \times 2} & M\bI_2 
        \end{bmatrix} \implies (\bK_1^h \bZ_0)_{:,k} = \begin{bmatrix} M\sin(\theta_k) \\ M\cos(\theta_k) \end{bmatrix}, \quad \forall h \in \{P+1, P+2\}, \\
        \bV_1^{P+1} &= \begin{bmatrix} 
            \bm{0}_{1 \times d} & 1 & \bm{0}_{1 \times 3} \\ 
            \bm{0}_{1 \times d} & 0 & \bm{0}_{1 \times 3} 
        \end{bmatrix} \implies (\bV_1^{P+1} \bZ_0)_{:,k} = \begin{bmatrix} \delta_{k,1} \\ 0 \end{bmatrix}, \\
        \bV_1^{P+2} &= \begin{bmatrix} 
            \bm{0}_{2 \times d} & \bm{0}_{2 \times 2} & \bI_2 
        \end{bmatrix} \implies (\bV_1^{P+2} \bZ_0)_{:,k} = \begin{bmatrix} \sin(\theta_k) \\ \cos(\theta_k) \end{bmatrix}.
    \end{align*}
    
    The shared unnormalized attention score is given by
    \begin{equation*}
        s_{k,j}^{\text{id}} = (\bK_1^h \bZ_0)_{:,k}^\top (\bQ_1^h \bZ_0)_{:,j} = M \big(\sin(\theta_j)\sin(\theta_k) + \cos(\theta_j)\cos(\theta_k)\big) = M \cos(\theta_j - \theta_k).
    \end{equation*}
    The corresponding attention weight is $(\bA^{\text{id}})_{k,j} = \frac{\exp(M\cos(\theta_j - \theta_k))}{S^{\text{id}}}$. Due to the symmetry of the uniform periodic grid, the normalization denominator $S^{\text{id}}$ is invariant to the query index $j$. By applying a periodic shift to the summation index, it simplifies to
    \begin{equation*}
        S^{\text{id}} = \sum_{l=1}^P \exp\big(M\cos(\theta_j - \theta_l)\big) = \sum_{l=1}^P \exp\big(M\cos(\theta_l)\big).
    \end{equation*}
    
    For the indicator head ($h=P+1$), to align with $d_v=2$, we pad the indicator value with a zero row, making the value vector $[\delta_{k,1}, 0]^\top$. Thus, the output for the $j$-th token is $[\tilde{I}_j, 0]^\top$, where the first component $\tilde{I}_j$ evaluates to
    $$\tilde{I}_j = \sum_{k=1}^P (\bA^{\text{id}})_{k,j} \delta_{k,1} = (\bA^{\text{id}})_{1,j} = \frac{e^{M\cos(\theta_j - \theta_1)}}{\sum_{l=1}^P e^{M\cos(\theta_l)}}.$$
    For the source token ($j=1$), the numerator is $e^M$. We define this exact output as a deterministic scalar coefficient $\eta(M)$ that
    $$\tilde{I}_1 = \eta(M) := \frac{e^M}{\sum_{l=1}^P e^{M\cos(\theta_l)}}.$$
    By bounding the $P-1$ off-diagonal terms in the denominator by $e^{M(1-c)}$, we get
    $$\eta(M) \ge \frac{e^M}{e^M + (P-1)e^{M(1-c)}} \implies 0 \le 1 - \eta(M) \le (P-1)e^{-cM}.$$
    For non-source tokens ($j \neq 1$), the numerator is bounded by $e^{M(1-c)}$ while the denominator is at least $e^M$, we have
    $$\tilde{I}_j \le \frac{e^{M(1-c)}}{e^M} = e^{-cM}.$$
    This deterministic scalar $\eta(M)$ allows the subsequent point-wise FFN layer to restore the indicator to  $1$.

    For the PE head ($h=P+2$), the value vector is $[\sin(\theta_k), \cos(\theta_k)]^\top$. By substituting the relative phase $\phi_k = \theta_k - \theta_j$, the output for the $j$-th token becomes
    \begin{align*}
        \sum_{k=1}^P (\bA^{\text{id}})_{k,j} \begin{bmatrix} \sin(\theta_k) \\ \cos(\theta_k) \end{bmatrix} 
        &= \frac{1}{S^{\text{id}}} \sum_{k=1}^P e^{M \cos(\phi_k)} \begin{bmatrix} \sin(\phi_k + \theta_j) \\ \cos(\phi_k + \theta_j) \end{bmatrix} \\
        &= \frac{1}{S^{\text{id}}} \sum_{k=1}^P e^{M \cos(\phi_k)} \begin{bmatrix} \sin(\phi_k)\cos(\theta_j) + \cos(\phi_k)\sin(\theta_j) \\ \cos(\phi_k)\cos(\theta_j) - \sin(\phi_k)\sin(\theta_j) \end{bmatrix}.
    \end{align*}
    Since subtracting the grid point $\theta_j$ applies a cyclic permutation, the relative phases reconstruct the identical equispaced grid of base angles. Specifically, evaluated as an unordered set modulo $2\pi$, we have the set equality: $\Phi = \{\phi_k \pmod{2\pi}\}_{k=1}^P = \left\{\frac{2\pi (k-1)}{P}\right\}_{k=1}^{P} = \{\theta_k\}_{k=1}^P$. Defining the summand $g(\phi) = e^{M \cos(\phi)} \sin(\phi)$, we partition the grid $\Phi$ into self-conjugate poles $\Phi_0 = \Phi \cap \{\pi, 2\pi\}$ (where $\pi \in \Phi$ if and only if $P$ is even) and the remaining conjugate pairs $\{\phi, 2\pi - \phi\}$. 
    For $\phi \in \Phi_0$, the term disappears trivially since $g(0) = g(\pi) = 0$. For any conjugate pair, we have
    $$ g(2\pi - \phi) + g(\phi) = e^{M \cos(\phi)}\big(-\sin(\phi) + \sin(\phi)\big) = 0. $$
    Thus, regardless of the parity of $P$, the global summation becomes
    $$ \sum_{k=1}^P e^{M \cos(\phi_k)} \sin(\phi_k) = \sum_{\phi \in \Phi_0} g(\phi) + \frac{1}{2} \sum_{\phi \notin \Phi_0} \big(g(\phi) + g(2\pi-\phi)\big) = 0. $$
    Hence, utilizing the established set equality $\Phi = \{\theta_l\}_{l=1}^P$, the output for $j$-th token simplifies to a scaled version of the input positional encoding, that is
    \begin{equation*}
        \left( \frac{\sum_{l=1}^P e^{M \cos(\theta_l)} \cos(\theta_l)}{\sum_{l=1}^P e^{M \cos(\theta_l)}} \right) \begin{bmatrix} \sin(\theta_j) \\ \cos(\theta_j) \end{bmatrix} = \lambda(M) \begin{bmatrix} \sin(\theta_j) \\ \cos(\theta_j) \end{bmatrix}.
    \end{equation*}
    Considering $1 - \lambda(M) = \frac{\sum_{l=1}^P e^{M \cos(\theta_l)} (1 - \cos(\theta_l))}{\sum_{l=1}^P e^{M \cos(\theta_l)}}$, for the peak term where $\theta_l = 0$ (corresponding to $l=P$), the numerator term is $e^M(1 - \cos(0)) = 0$. For the remaining $P-1$ terms ($l \neq P$), we have $e^{M \cos(\theta_l)} \le e^{M(1-c)}$ and $(1 - \cos(\theta_l)) \le 2$. Lower bounding the denominator by the single peak term $e^M$, we obtain
    \begin{equation*}
        0 \le 1 - \lambda(M) \le \frac{(P-1) \cdot e^{M(1-c)} \cdot 2}{e^M} = 2(P-1)e^{-cM}.
    \end{equation*}
    Thus, positional encodings are preserved with a uniform scaling coefficient $\lambda(M)$, allowing easy restoration in the following point-wise FFN layer.

   \textbf{Step 3: Output Projection and Error Bounds.}    
    Concatenating the outputs of all $P+2$ heads yields the intermediate matrix $\bC \in \RR^{2(P+2) \times P}$ in the following form
    \begin{equation*}
        \bC = \begin{bmatrix} 
            T_1(\bx) & \epsilon_{1,2} & \dots & \epsilon_{1,P} \\ 
            R_1(\bx) & \tilde{\epsilon}_{1,2} & \dots & \tilde{\epsilon}_{1,P} \\ 
            \epsilon_{2,1} & T_2(\bx) & \dots & \epsilon_{2,P} \\ 
            \tilde{\epsilon}_{2,1} & R_2(\bx) & \dots & \tilde{\epsilon}_{2,P} \\ 
            \vdots & \vdots & \ddots & \vdots \\ 
            \epsilon_{P,1} & \epsilon_{P,2} & \dots & T_P(\bx) \\ 
            \tilde{\epsilon}_{P,1} & \tilde{\epsilon}_{P,2} & \dots & R_P(\bx) \\ 
            \tilde{I}_1 & \tilde{I}_2 & \dots & \tilde{I}_P \\
            0 & 0 & \dots & 0 \\
            \lambda(M)\sin(\theta_1) & \lambda(M)\sin(\theta_2) & \dots & \lambda(M)\sin(\theta_P) \\
            \lambda(M)\cos(\theta_1) & \lambda(M)\cos(\theta_2) & \dots & \lambda(M)\cos(\theta_P)
        \end{bmatrix}.
    \end{equation*}
    By introducing the block identity matrix $\begin{bmatrix} \bI_2 & \dots & \bI_2 \end{bmatrix} \in \RR^{2 \times 2P}$, we choose the output projection matrix $\bW^O \in \RR^{D \times 2(P+2)}$ as
    \begin{equation*}
        \bW^O = \begin{bmatrix} 
            \begin{bmatrix} \bI_2 & \dots & \bI_2 \end{bmatrix} & \bm{0}_{2 \times 2} & \bm{0}_{2 \times 2} \\ 
            \bm{0}_{(d-1) \times 2P} & \bm{0}_{(d-1) \times 2} & \bm{0}_{(d-1) \times 2} \\ 
            \bm{0}_{1 \times 2P} & \begin{bmatrix}1 & 0 \end{bmatrix} & \bm{0}_{1 \times 2} \\ 
            \bm{0}_{2 \times 2P} & \bm{0}_{2 \times 2} & \bI_2 
        \end{bmatrix},
    \end{equation*}
    the MHA output for the $j$-th token is
    \begin{equation*}
        (\widehat{\bZ}_1)_{:, j} = \bW^O \bC_{:,j} = 
        \begin{bmatrix} 
            \tilde{T}_j(\bx) \\ 
            \tilde{R}_j(\bx) \\
            \bm{0}_{d-1} \\ 
            \tilde{I}_j \\ 
            \lambda(M) \begin{bmatrix} \sin(\theta_j) \\ \cos(\theta_j) \end{bmatrix}
        \end{bmatrix},
    \end{equation*}    
    where
    \begin{equation*}
        \tilde{T}_j(\bx) = T_j(\bx) + \sum_{h \neq j} \epsilon_{h,j}, \quad \tilde{R}_j(\bx) = R_j(\bx) + \sum_{h \neq j} \tilde{\epsilon}_{h,j}.
    \end{equation*}   
    Given the uniform bounds $|\epsilon_{h,j}| \le e^{-cM} B_T$ and $|\tilde{\epsilon}_{h,j}| \le e^{-cM} B_R$, we have
    \begin{align*}
        |\tilde{T}_j(\bx) - T_j(\bx)| &= \Big| \sum_{h \neq j} \epsilon_{h,j} \Big| \le \sum_{h \neq j} |\epsilon_{h,j}| \le (P-1) e^{-cM} B_T, \\
        |\tilde{R}_j(\bx) - R_j(\bx)| &= \Big| \sum_{h \neq j} \tilde{\epsilon}_{h,j} \Big| \le \sum_{h \neq j} |\tilde{\epsilon}_{h,j}| \le (P-1) e^{-cM} B_R.
    \end{align*}
    
    \textbf{Step 4: Parameter Bounds.}
    To bound the parameter magnitude, we evaluate the maximum absolute entry $\|\cdot\|_{\max}$ for each parameter matrix over all heads $h \in [P+2]$. The peak attention weight satisfies $\eta(M)^{-1} \le \frac{e^M + (P-1)e^{M(1-c)}}{e^M} = 1 + (P-1)e^{-cM}$. It follows that the maximum entries are explicitly bounded by
    \begin{align*}
        \max_{h} \|\bQ_1^h\|_{\max} &\le 1, \\
        \max_{h} \|\bK_1^h\|_{\max} &= M, \\
        \max_{h} \|\bV_1^h\|_{\max} &\le \max\big\{1, \eta(M)^{-1} B_{\cA_1} \big\} \le \max\big\{1, (1 + (P-1)e^{-cM}) B_{\cA_1}\big\}, \\
        \|\bW^O\|_{\max} &= 1.
    \end{align*}
    Therefore, the maximum absolute value of any parameter in the MHA layer $\cA_1$, denoted as $M_{\cA_1}$, is bounded by
    \begin{align*}
        M_{\cA_1} &= \max \Big\{ \max_h \|\bQ_1^h\|_{\max}, \max_h \|\bK_1^h\|_{\max}, \max_h \|\bV_1^h\|_{\max}, \|\bW^O\|_{\max} \Big\} \\
        &\le \max \big\{ 1, M, (1 + (P-1)e^{-cM}) B_{\cA_1} \big\}.
    \end{align*}
    Thus, we complete the proof.
\end{proof}

The first point-wise FFN layer acts as a precise non-linear denoiser, leveraging the hard-thresholding property of ReLU activations to completely eliminate small errors and recover the canonical form and original sinusoidal positional encodings.
\begin{lemma}[FFN Exact Thresholding and Restoration] \label{lemma_FFN_restore}
    Let $\widehat{\bZ}_1 \in \RR^{D \times P}$ (where $D = d+4$) be the output matrix from the MHA layer defined in Lemma \ref{lemma_mha_affine}. Assume $M$ is sufficiently large such that $(P+1)e^{-cM} < 1$. There exists a point-wise FFN layer $\cF_1: \RR^{D \times P} \to \RR^{D \times P}$ with hidden dimension $d_{\text{ff}} = 2D = 2d+8$, such that its output $\bZ_1 = \cF_1(\widehat{\bZ}_1) \in \RR^{D \times P}$ takes the column-wise form
    \begin{equation*}
        (\bZ_1)_{:, j} = \begin{bmatrix}
            \tilde{T}_j(\bx) \\
            \tilde{R}_j(\bx) \\
            \bm{0}_{d-1} \\
            \delta_{j,1} \\
            \sin(\theta_j) \\
            \cos(\theta_j)
        \end{bmatrix}, \quad \forall j \in [P].
    \end{equation*}
    The maximum parameter magnitude $M_{\cF_1}$ satisfy
    \begin{align*}
        M_{\cF_1} \le \frac{1}{1 - 2P e^{-cM}}.
    \end{align*}
\end{lemma}

\begin{proof}[Proof of Lemma \ref{lemma_FFN_restore}]
    Let $w_{\text{ind}} = \frac{1}{\eta(M) - 2e^{-cM}}$. We define the $j$-th column of the input as $\hat{\bz}_j = (\widehat{\bZ}_1)_{:,j}$ and the $j$-th column of the output as $\bz_j = (\bZ_1)_{:, j}$. We construct the FFN applied to each column as $\bz_j = \bW_2 \sigma(\bW_1 \hat{\bz}_j + \bb_1)$. 
    
    To uniformly map the $D$-dimensional input across the $2D$ hidden neurons, we introduce a diagonal scaling matrix $\bLambda \in \RR^{D \times D}$ defined as
    \begin{equation*}
        \bLambda = \mathrm{diag}\Big(\underbrace{1, \dots, 1}_{d+1}, w_{\text{ind}}, \frac{1}{\lambda(M)}, \frac{1}{\lambda(M)}\Big).
    \end{equation*}
    We set the weights $\bW_1 \in \RR^{2D \times D}$, $\bW_2 \in \RR^{D \times 2D}$ and biases $\bb_1 \in \RR^{2D}$, $\bb_2 \in \RR^{D}$ as
    \begin{align*}
        \bW_1 &= \begin{bmatrix}
            \bI_D \\
            -\bI_D
        \end{bmatrix}, \quad
        \bb_1 = \begin{bmatrix}
            -2e^{-cM} \be_{d+2} \\
            \bm{0}_D
        \end{bmatrix}, \\
        \bW_2 &= \begin{bmatrix}
            \bLambda & -\bLambda
        \end{bmatrix}, \quad \bb_2 = \bm{0}_D,
    \end{align*}
    where $\be_{d+2} \in \RR^D$ is the standard basis vector for the $(d+2)$-th component.
    
    Using the ReLU identity $x = \sigma(x) - \sigma(-x)$, we analyze the output column $\bz_j$. By evaluating $\bz_j = \bLambda \sigma(\hat{\bz}_j - 2e^{-cM}\be_{d+2}) - \bLambda \sigma(-\hat{\bz}_j)$ component-wise, we get
    \begin{align*}
        (\bz_j)_{1:d+1} &= \sigma((\hat{\bz}_j)_{1:d+1}) - \sigma(-(\hat{\bz}_j)_{1:d+1})  = (\hat{\bz}_j)_{1:d+1}, \\
        (\bz_j)_{d+3:d+4} &= \frac{1}{\lambda(M)} \big( \sigma((\hat{\bz}_j)_{d+3:d+4}) - \sigma(-(\hat{\bz}_j)_{d+3:d+4}) \big) = \frac{1}{\lambda(M)} (\hat{\bz}_j)_{d+3:d+4}.
    \end{align*}
    For the indicator neuron (index $d+2$), since $\tilde{I}_j \ge 0$, the negative channel trivially evaluates to $\sigma(-\tilde{I}_j) = 0$. Thus, substituting $(\hat{\bz}_j)_{d+2} = \tilde{I}_j$ yields
    \begin{align*}
        (\bz_j)_{d+2} &= w_{\text{ind}} \cdot \sigma(\tilde{I}_j - 2e^{-cM}) - 0 \\
        &= \begin{cases}
            w_{\text{ind}} \cdot \sigma(\eta(M) - 2e^{-cM}) = 1, & \text{if } j = 1, \\
            w_{\text{ind}} \cdot \sigma(\tilde{I}_j - 2e^{-cM}) = 0, & \text{if } j \neq 1 \text{ (since } \tilde{I}_j \le e^{-cM}),
        \end{cases}
    \end{align*}
    yielding $(\bz_j)_{d+2} = \delta_{j,1}$. Concatenating all evaluated components, the exact restored output column is precisely $\bz_j = [\tilde{T}_j(\bx), \tilde{R}_j(\bx), \bm{0}_{d-1}^\top, \delta_{j,1}, \sin(\theta_j), \cos(\theta_j)]^\top$.

    Using the bounds $\lambda(M) \ge 1 - 2(P-1)e^{-cM}$ and $\eta(M) \ge 1 - (P-1)e^{-cM}$ from Lemma \ref{lemma_mha_affine}, the parameter magnitude is bounded by the maximum entry in $\bLambda$ and $\bb_1$:
    \begin{align*}
        M_{\cF_1} &= \max \Big\{1, \frac{1}{\lambda(M)}, \frac{1}{\eta(M) - 2e^{-cM}}, 2e^{-cM} \Big\} \\
        &\le \max \Big\{ \frac{1}{1 - 2(P-1)e^{-cM}}, \frac{1}{1 - (P+1)e^{-cM}} \Big\} \le \frac{1}{1 - 2P e^{-cM}},
    \end{align*}
    where the constants $1$ and $2e^{-cM}$ are absorbed since the upper bound is strictly greater than $1$.
\end{proof}

The second MHA layer aggregates local affine transformation features to a global approximant via softmax attention mechanism, utilizing $\tilde{T}_j(\bx)$ as local spatial features and $\tilde{R}_j(\bx)$ as local approximation expert features.
\begin{lemma}[Softmax POU via MHA] \label{lemma_mha_pou}
    Let $\bZ_1 \in \RR^{D \times P}$ (where $D = d+4$) be the output matrix from the FFN layer defined in Lemma \ref{lemma_FFN_restore}. There exists a single-head MHA layer $\cA_2: \RR^{D \times P} \to \RR^{D \times P}$ with $H=1$ and $d_k=d_v=1$, yielding output $\widehat{\bZ}_2 = \cA_2(\bZ_1) \in \RR^{D \times P}$ such that its first column ($k=1$) evaluates to
    \begin{equation*}
        (\widehat{\bZ}_2)_{:, 1} = \begin{bmatrix}
            \hat{g}_T(\bx) \\
            \bm{0}_{D-1}
        \end{bmatrix} = \hat{g}_T(\bx) \be_1,
    \end{equation*}
    where
    \begin{equation*}
        \hat{g}_T(\bx) := \sum_{j=1}^P \tilde{\beta}_j(\bx) \tilde{R}_j(\bx), \quad \tilde{\beta}_j(\bx) := \frac{\exp(\tilde{T}_j(\bx))}{\sum_{l=1}^P \exp(\tilde{T}_l(\bx))},
    \end{equation*}
    and all other columns ($k \neq 1$) evaluate to constant vectors $\tilde R(\bx) \be_1$, where $\tilde R(\bx) = \frac{1}{P} \sum_{j=1}^P \tilde{R}_j(\bx)$.
    The maximum parameter magnitude $M_{\cA_2}$ satisfy
    \begin{align*}
        M_{\cA_2} = 1.
    \end{align*}
\end{lemma}

\begin{proof}[Proof of Lemma \ref{lemma_mha_pou}]
    We use a single attention head ($H=1$) with $d_k = d_v = 1$. Let $\be_1, \be_2, \be_{d+2} \in \RR^D$ be the standard basis vectors. We define the projection matrices $\bQ_2^1, \bK_2^1, \bV_2^1 \in \RR^{1 \times D}$ as
    \begin{align*}
        \bQ_2^1 &= \be_{d+2}^\top \implies (\bQ_2^1 \bZ_1)_{:,j} = \delta_{j,1}, \\
        \bK_2^1 &= \be_1^\top \implies (\bK_2^1 \bZ_1)_{:,k} = \tilde{T}_k(\bx), \\
        \bV_2^1 &= \be_2^\top \implies (\bV_2^1 \bZ_1)_{:,k} = \tilde{R}_k(\bx).
    \end{align*}  
    The unnormalized attention score $s_{k,j} = (\bK_2^1 \bZ_1)_{:,k}^\top (\bQ_2^1 \bZ_1)_{:,j}$ evaluates to
    \begin{equation*}
        s_{k,j} = \tilde{T}_k(\bx) \cdot \delta_{j,1} = \begin{cases}
            \tilde{T}_k(\bx), & j = 1, \\
            0, & j \neq 1.
        \end{cases}
    \end{equation*}
    Applying the column-wise softmax function across the key indices $k \in [P]$ yields the attention weights $(\bA_2^1)_{k,j}$. Since the normalization denominator is $S_j = \sum_{l=1}^P \exp(s_{l,j})$, we have
    \begin{equation*}
        (\bA_2^1)_{k,j} = \frac{\exp(s_{k,j})}{\sum_{l=1}^P \exp(s_{l,j})} = \begin{cases}
            \frac{\exp(\tilde{T}_k(\bx))}{\sum_{l=1}^P \exp(\tilde{T}_l(\bx))} = \tilde{\beta}_k(\bx), & j = 1, \\
            \frac{\exp(0)}{\sum_{l=1}^P \exp(0)} = \frac{1}{P}, & j \neq 1.
        \end{cases}
    \end{equation*}  
    The aggregated scalar output for the $j$-th token is
    \begin{equation*}
        \bH_{:,j} = \sum_{k=1}^P (\bA_2^1)_{k,j} (\bV_2^1 \bZ_1)_{:,k}= \begin{cases}
            \sum_{k=1}^P \tilde{\beta}_k(\bx) \tilde{R}_k(\bx) = \hat{g}_T(\bx), & j = 1, \\
            \frac{1}{P} \sum_{k=1}^P \tilde{R}_k(\bx) := \tilde R(\bx), & j \neq 1.
        \end{cases}
    \end{equation*}  
    We define the output projection matrix $\bW^O \in \RR^{D \times 1}$ as $\bW^O = \be_1$. The output columns are
    \begin{equation*}
        (\widehat{\bZ}_2)_{:, j} = \bW^O \bH_{:,j} = \begin{cases}
            \hat{g}_T(\bx) \be_1, & j = 1, \\
            \tilde R(\bx) \be_1, & j \neq 1.
        \end{cases}
    \end{equation*}
    Since all entries in the projection matrices are canonical basis elements, the maximum parameter magnitude is bounded by $M_{\cA_2} = 1$. Thus we complete the proof.
\end{proof}

\section{Proof of approximation results in Theorems \ref{thm_holder_approx} and \ref{thm_manifold_approx}}
\label{sec:appendixB}

This section contains the comprehensive proofs for our main approximation guarantees. By assembling the foundational Transformer construction lemmas established in Subsection \ref{sec:appendixA3}, we construct the full architectures required for the target functions.
We need the following lemma from Corollary A.7 in \cite{edelman2022inductive}.
\begin{lemma} \label{Softmaxlipchitz}
    For any $\btheta, \btheta' \in \RR^d$, we have
    \begin{equation*}
        \left\| \softmax(\btheta)- \softmax(\btheta') \right\|_1 \leq 2 \|\btheta- \btheta'\|_\infty.
    \end{equation*}
\end{lemma}

\subsection{Approximation proof in Theorem \ref{thm_holder_approx}}
\label{sec:appendixB1}

We provide the detailed construction and error bound analysis for Theorem \ref{thm_holder_approx}. This formally establishes the uniform $\varepsilon$-approximation error for H\"older continuous functions defined on Euclidean domains.
\begin{proof}[Proof of Theorem \ref{thm_holder_approx}]
    Applying Lemma \ref{lemma_POU_approx} with target accuracy $\varepsilon/2$ yields $P$ centers $\{\bc_j\}_{j=1}^P \subset [0,1]^d$ with
    \begin{equation*}
        P \le \big(\sqrt{d}(4C_H)^{1/\alpha}\big)^d (\varepsilon/2)^{-\frac{d}{\alpha}} = C_P \varepsilon^{-\frac{d}{\alpha}}, \quad \text{where } C_P = \big(\sqrt{d}(8C_H)^{1/\alpha}\big)^d.
    \end{equation*}
    Define the target affine features $T_j(\bx) := 2M_g \bc_j^\top \bx - M_g \|\bc_j\|_2^2$ and $R_j(\bx) := g(\bc_j)$. The corresponding Softmax POU approximation is given by
    \begin{align}
        \hat{g}(\bx) &= \sum_{j=1}^P \beta_j(\bx) g(\bc_j) = \sum_{j=1}^P \frac{\exp(T_j(\bx))}{\sum_{l=1}^P \exp(T_l(\bx))} g(\bc_j).
    \end{align}
    By Lemma \ref{lemma_POU_approx}, this approximation satisfies 
    $$\sup_{\bx \in [0,1]^d} |\hat{g}(\bx) - g(\bx)| \le \varepsilon/2,$$
    provided that
    \begin{equation*}
        M_g = C_M \varepsilon^{-\frac{2}{\alpha}} \log \frac{2}{\varepsilon}, \quad \text{with } C_M = \frac{(8C_H)^{2/\alpha}}{3} \left(\log\big( 4B \big(\sqrt{d}(4C_H)^{1/\alpha}\big)^d \big) + \frac{d+\alpha}{\alpha} \right).
    \end{equation*}
    Furthermore, since $\bx, \bc_j \in [0,1]^d$, the magnitudes of these target features are bounded by
    \begin{align*}
        B_T &:= \max_{j \in [P]} \|T_j\|_\infty \le 2M_g \sup_{\bx \in [0,1]^d} |\bc_j^\top \bx| + M_g \|\bc_j\|_2^2 \le 3M_g d, \\
        B_R &:= \max_{j \in [P]} \|R_j\|_\infty = \max_{j \in [P]} |g(\bc_j)| \le \|g\|_\infty \le B.
    \end{align*}

    \textbf{Step 1: Transformer Construction.}
    By Lemma \ref{lemma_preprocessing}, the input $\bx$ is mapped to the initial sequence matrix $\bZ_0 = \cP(\bx) \in \RR^{D \times P}$ structured as
    \begin{equation*}
        \bZ_0 = \begin{bmatrix}
            \bx & \bm{0} & \cdots & \bm{0} \\
            1 & 0 & \cdots & 0 \\
            0 & 0 & \cdots & 0 \\
            \sin(\theta_1) & \sin(\theta_2) & \cdots & \sin(\theta_P) \\
            \cos(\theta_1) & \cos(\theta_2) & \cdots & \cos(\theta_P)
        \end{bmatrix}.
    \end{equation*}
    
    For Encoder Block 1, applying the MHA layer $\cA_1$ from Lemma \ref{lemma_mha_affine} yields $\widehat{\bZ}_1 = \cA_1(\bZ_0) \in \RR^{D \times P}$, where the $j$-th column evaluates to
    \begin{equation*}
        (\widehat{\bZ}_1)_{:, j} = \big[ \tilde{T}_j(\bx), \tilde{g}(\bc_j), \bm{0}_{d-1}^\top, \tilde{I}_j, \lambda(M) \sin(\theta_j), \lambda(M) \cos(\theta_j) \big]^\top, \quad \forall j \in [P],
    \end{equation*}
    where $\tilde{I}_1 = \eta(M) := \frac{e^M}{\sum_{l=1}^P e^{M\cos(\theta_l)}}$, and
    \begin{align*}
        0 &\le 1 - \tilde{I}_1 = 1 - \eta(M) \le (P-1)e^{-cM}, && \text{for } j = 1, \\
        0 &\le \tilde{I}_j \le e^{-cM}, && \text{for } j \neq 1.
    \end{align*}
    The features $\tilde{T}_j(\bx)$ and $\tilde{g}(\bc_j)$ approximate the target affine expressions with bounded errors
    \begin{align*}
        &|\tilde{T}_j(\bx) - T_j(\bx)| \le 3(P-1) M_g d \, e^{-cM} , \\
        &|\tilde{g}(\bc_j) - g(\bc_j)| \le (P-1) B e^{-cM} .
    \end{align*}
    Next, applying the FFN layer $\cF_1$ from Lemma \ref{lemma_FFN_restore} restores the structural components, yielding $\bZ_1 = \cF_1(\widehat{\bZ}_1)$ with columns
    \begin{equation*}
        (\bZ_1)_{:, j} = \big[ \tilde{T}_j(\bx), \tilde{g}(\bc_j), \bm{0}_{d-1}^\top, \delta_{j,1}, \sin(\theta_j), \cos(\theta_j) \big]^\top, \quad \forall j \in [P].
    \end{equation*}
    
   For Encoder Block 2, applying the MHA layer $\cA_2$ from Lemma \ref{lemma_mha_pou} that computes attention scores over $\tilde{T}_j(\bx)$ and aggregates local function approximations $\tilde{g}(\bc_j)$. The output $\widehat{\bZ}_2 = \cA_2(\bZ_1)$ has its first column to be
    \begin{equation*}
        (\widehat{\bZ}_2)_{:, 1} = \big[ \hat{g}_T(\bx), \bm{0}_{D-1}^\top \big]^\top = \hat{g}_T(\bx) \be_1,
    \end{equation*}
    where
    \begin{equation*}
        \hat{g}_T(\bx) = \sum_{j=1}^P \tilde{\beta}_j(\bx) \tilde{g}(\bc_j), \quad \tilde{\beta}_j(\bx) = \frac{\exp(\tilde{T}_j(\bx))}{\sum_{l=1}^P \exp(\tilde{T}_l(\bx))},
    \end{equation*}
    all other columns ($k \neq 1$) evaluate to constant vectors $\tilde R(\bx) \be_1$, where $\tilde R(\bx) = \frac{1}{P} \sum_{j=1}^P \tilde{R}_j(\bx)$.
    The second point-wise FFN layer $\cF_2$ preserves this output via an exact identity mapping, exploiting the ReLU property $x = \sigma(x) - \sigma(-x)$. Specifically, we set the hidden dimension $d_{\text{ff}} = 2D = 2d+8$, and configure the weights ($\bW_1 \in \RR^{2D \times D}, \bW_2 \in \RR^{D \times 2D}$) and biases ($\bb_1 \in \RR^{2D}, \bb_2 \in \RR^D$) as
    \begin{equation*}
        \bW_1 = \begin{bmatrix} \bI_D \\ -\bI_D \end{bmatrix}, \quad \bb_1 = \bm{0}_{2D}, \quad \bW_2 = \begin{bmatrix} \bI_D & -\bI_D \end{bmatrix}, \quad \bb_2 = \bm{0}_D.
    \end{equation*}
    This yields $\bZ_2 = \cF_2(\widehat{\bZ}_2) = \bW_2 \sigma(\bW_1 \widehat{\bZ}_2) = \widehat{\bZ}_2$. Finally, the scalar regression output is then obtained by reading out the relevant components of $\bZ_2$, via a linear affine mapping $\bc_{3}^\top \mathrm{vec}(\bZ_2)$, where setting $\bc_3 = [1, 0, \dots, 0]^\top \in \RR^{DP}$ explicitly extracts $\hat{g}_T(\bx)$.

    \textbf{Step 2: Error Bound Analysis.}
    According to Lemma \ref{Softmaxlipchitz},
    \begin{equation*}
        \sum_{j=1}^P |\tilde{\beta}_j(\bx) - \beta_j(\bx)| \le 2 \max_{j \in [P]} |\tilde{T}_j(\bx) - T_j(\bx)| \le 6(P-1) M_g d \, e^{-cM}.
    \end{equation*}
    It follows that the approximation error of the Transformer is bounded by
    \begin{align*}
        |\hat{g}_T(\bx) - \hat{g}(\bx)| &= \Big| \sum_{j=1}^P \tilde{\beta}_j(\bx) \tilde{g}(\bc_j) - \sum_{j=1}^P \beta_j(\bx) g(\bc_j) \Big| \\
        &\le \sum_{j=1}^P \tilde{\beta}_j(\bx) |\tilde{g}(\bc_j) - g(\bc_j)| + \sum_{j=1}^P |\tilde{\beta}_j(\bx) - \beta_j(\bx)| |g(\bc_j)| \\
        &\le (P-1) B e^{-cM} + 6(P-1) M_g d \, e^{-cM} \cdot B \\
        &= (P-1) B (1+6M_g d) e^{-cM}.
    \end{align*}
    By setting the scaling factor $M = \frac{1}{c} \log \big( \frac{2(P-1)B(1 + 6 M_g d)}{\varepsilon} \big)$, the implementation error is strictly bounded by $\varepsilon/2$. Combining this with the approximation error of the ideal Softmax POU, the triangle inequality yields
    \begin{equation*}
        \sup_{\bx \in [0,1]^d} |\hat{g}_T(\bx) - g(\bx)| \le \sup_{\bx \in [0,1]^d} |\hat{g}_T(\bx) - \hat{g}(\bx)| + \sup_{\bx \in [0,1]^d} |\hat{g}(\bx) - g(\bx)| \le \frac{\varepsilon}{2} + \frac{\varepsilon}{2} = \varepsilon.
    \end{equation*}

    \textbf{Step 3: Parameter Complexity.}
    We count the total architectural (dense) parameters $\cN_{total}$ based on the embedding dimension $D = d+4$:
    \begin{itemize}
        \item \textbf{Pre-processing}: Affine mapping $\RR^d \to \RR^{D \times P}$ has $D(P+d+1)$ parameters.
        \item \textbf{MHA$_1$}: $H=P+2$ heads with $d_k=d_v=2$ have $8(P+2)D$ parameters.
        \item \textbf{FFN$_1$}: Hidden dimension $2D$ has $4D^2 + 3D = 4d^2 + 35d + 76$ parameters.
        \item \textbf{MHA$_2$}: $H=1$ head with $d_k=d_v=1$ has $4D$ parameters.
        \item \textbf{FFN$_2$}: Hidden dimension $2D$ has $4D^2 + 3D = 4d^2 + 35d + 76$ parameters.
        \item \textbf{Readout}: Linear projection over $\mathrm{vec}(\bZ_2)$ has $DP$ parameters.
    \end{itemize}
    Summing these contributions yields:
    \begin{align*}
        \cN_{total} &= D(P+d+1) + 8D(P+2) + (4d^2 + 35d + 76) + 4D + (4d^2 + 35d + 76) + DP \\
        &= 10P(d+4) + 9d^2 + 95d + 236.
    \end{align*}
    Since $P \le C_P \varepsilon^{-d/\alpha}$ and $\varepsilon \le 1/e \implies \varepsilon^{-d/\alpha} > 1$, we obtain
    \begin{equation*}
        \cN_{total} \le \big[ 10(d+4) C_P + 9d^2 + 95d + 236 \big] \varepsilon^{-d/\alpha} := C_N \varepsilon^{-d/\alpha}.
    \end{equation*}

    \textbf{Step 4: Parameter Bounds.}
    According to our construction, to bound the global parameter magnitude $M_{max} \le \max\big\{1, M, \frac{1}{1 - 2Pe^{-cM}}, (1 + (P-1)e^{-cM}) B_{\cA_1} \big\}$, we first bound $M$. Since $P \ge 2$, we have $\sin(\pi/P) \ge 2/P$, yielding
    \begin{equation*}
        \frac{1}{c} = \frac{1}{2\sin^2(\pi/P)} \le \frac{P^2}{8} \le \frac{C_P^2}{8} \varepsilon^{-2d/\alpha}.
    \end{equation*}
    For $\varepsilon \le 1/e$, substituting $P \le C_P \varepsilon^{-d/\alpha}$ and $M_g \le C_M \varepsilon^{-2/\alpha} \log(2/\varepsilon)$ yields
    \begin{align*}
        \log \Big( \frac{2(P-1)B(1 + 6 M_g d)}{\varepsilon} \Big) &\le \log \left( \frac{2 C_P B (1 + 6 d C_M)}{\varepsilon^{(d+2+\alpha)/\alpha}} \log\frac{2}{\varepsilon} \right) \le C_{log} \log \frac{1}{\varepsilon},
    \end{align*}
    where $C_{log} = \big| \log(4 C_P B (1 + 6 d C_M)) \big| + \frac{d+2+\alpha}{\alpha} + 2$. It follows that
    \begin{equation*}
        M = \frac{1}{c} \log \Big( \frac{4(P-1)B(1 + 6 M_g d)}{\varepsilon} \Big) \le C_{mag} \varepsilon^{-2d/\alpha} \log \frac{1}{\varepsilon}, \quad \text{where } C_{mag} = \frac{C_P^2 C_{log}}{8}.
    \end{equation*}
    Next, by definition $e^{-cM} = \frac{\varepsilon}{2(P-1)B(1 + 6 M_g d)}$. Since $1 + 6M_g d > 2 \geq \frac{P}{P-1}$, we have
    \begin{equation*}
        2P e^{-cM} = \frac{2P}{P-1} \frac{\varepsilon}{2B(1 + 6 M_g d)} \le \frac{\varepsilon}{B} \implies \frac{1}{1 - 2Pe^{-cM}} \le \frac{1}{1 - \frac{\varepsilon}{B}} \le \frac{1}{1 - \frac{1}{eB}}.
    \end{equation*}
    Moreover, since $B_{\cA_1} \le \max\{3M_g d, B, 1\} \le C_B \varepsilon^{-2/\alpha} \log(2/\varepsilon)$ with $C_B = \max\{3d C_M, B, 1\}$. Similarly bounding $(P-1)e^{-cM} \le \frac{\varepsilon}{2B}$, it follows that
    \begin{equation*}
        (1 + (P-1)e^{-cM}) B_{\cA_1} \le \Big(1 + \frac{\varepsilon}{2B}\Big) C_B \varepsilon^{-2/\alpha} \log \frac{2}{\varepsilon} \le 2 \Big(1 + \frac{1}{2eB}\Big) C_B \varepsilon^{-2d/\alpha} \log \frac{1}{\varepsilon}.
    \end{equation*}
    Consequently, the global parameter magnitude satisfies
    \begin{align} \label{tildeCmag}
        \nonumber M_{max} &\le \max \bigg\{ \frac{1}{1 - \frac{1}{eB}}, C_{mag} \varepsilon^{-2d/\alpha} \log \frac{1}{\varepsilon}, 2 \Big(1 + \frac{1}{2eB}\Big) C_B \varepsilon^{-2d/\alpha} \log \frac{1}{\varepsilon} \bigg\} \nonumber \\
        &\le \max\bigg\{ \frac{1}{1 - \frac{1}{eB}}, C_{mag}, 2 \Big(1 + \frac{1}{2eB}\Big) C_B \bigg\} \varepsilon^{-2d/\alpha} \log \frac{1}{\varepsilon} := \widetilde{C}_{mag} \varepsilon^{-2d/\alpha} \log \frac{1}{\varepsilon},
    \end{align}
    This completes the proof.
\end{proof}

\subsection{Approximation proof in Theorem \ref{thm_manifold_approx}} \label{sec:appendixB2}

In this subsection, we prove Theorem \ref{thm_manifold_approx} by adapting our constructive Transformer framework in Subsection \ref{sec:appendixB1} to Riemannian manifolds.
\begin{proof}[Proof of Theorem \ref{thm_manifold_approx}]
    Applying Lemma \ref{lemma_POU_approx_manifold} (Softmax POU Approximation on Manifolds) with target accuracy $\varepsilon/2$ yields $P = C_g \le C_P \varepsilon^{-d/\alpha}$ centers $\{\bc_j\}_{j=1}^P \subset \mathcal{M}$, where $C_P = C_{\mathcal{M}}(16C_H)^{d/\alpha}$. Define the target affine features $T_j(\bx) := 2M_g c_j^\top \bx - M_g \|\bc_j\|_2^2$ and $R_j(\bx) := g(\bc_j)$. The corresponding Softmax POU approximation is given by
    \begin{equation*}
        \hat{g}(\bx) = \sum_{j=1}^P \beta_j(\bx) g(\bc_j) = \sum_{j=1}^P \frac{\exp(T_j(\bx))}{\sum_{l=1}^P \exp(T_l(\bx))} g(\bc_j).
    \end{equation*}
    By Lemma \ref{lemma_POU_approx_manifold}, this achieves 
    $$\sup_{\bx \in \mathcal{M}} |\hat{g}(\bx) - g(\bx)| \le \varepsilon/2$$
    provided that
    \begin{equation*}
        M_g = C_M \varepsilon^{-2/\alpha} \log \frac{2}{\varepsilon}, \quad \text{where } C_M = \frac{(16C_H)^{2/\alpha}}{3} \left( \log\left(4 B C_P \right) + \frac{d+\alpha}{\alpha} \right).
    \end{equation*}
    Since $\bx, \bc_j \in \mathcal{M} \subseteq [0,1]^{\bar{d}}$, the magnitudes of target features are  bounded by $B_T := \max_{j\in[P]} \|T_j\|_\infty \le 3M_g \bar{d}$ and $B_R := \max_{j\in[P]} \|R_j\|_\infty \le B$.

    \textbf{Step 1: Transformer Construction.}
    Following the exact architecture construction as in the proof of Theorem \ref{thm_holder_approx}, but operating on the ambient dimension $\bar{d}$ instead of $d$, the pre-processing layer maps the input $\bx \in \mathbb{R}^{\bar{d}}$ to $\bZ_0 \in \mathbb{R}^{D \times P}$ using the embedding dimension $D = \bar{d} + 4$. The first encoder block (MHA$_1$ and FFN$_1$) achieves local features and restores the structural components to exact canonical forms, yielding $\bZ_1 \in \mathbb{R}^{D \times P}$ with columns
    \begin{equation*}
        (\bZ_1)_{:,j} = [\tilde{T}_j(\bx), \tilde{g}(\bc_j), \bm{0}_{\bar{d}-1}^\top, \delta_{j,1}, \sin(\theta_j), \cos(\theta_j)]^\top, \quad \forall j \in [P].
    \end{equation*}
    The local features approximate the targets with bounded errors 
    \begin{align*}
        |\tilde{T}_j(\bx) - T_j(\bx)| &\le 3(P-1)M_g \bar{d} e^{-cM}, \\
        |\tilde{g}(\bc_j) - g(\bc_j)| &\le (P-1)B e^{-cM}.
    \end{align*}
    The second encoder block (MHA$_2$ and FFN$_2$) computes the softmax attention and aggregates the values, and the readout extracts
    \begin{equation*}
        \hat{g}_T(\bx) = \sum_{j=1}^P \tilde{\beta}_j(\bx) \tilde{g}(\bc_j), \quad \tilde{\beta}_j(\bx) = \frac{\exp(\tilde{T}_j(\bx))}{\sum_{l=1}^P \exp(\tilde{T}_l(\bx))}.
    \end{equation*}

    \textbf{Step 2: Error Bound Analysis.}
    By the softmax Lipschitz property Lemma \ref{Softmaxlipchitz}, 
    $$\sum_{j=1}^P |\tilde{\beta}_j(\bx) - \beta_j(\bx)| \le 2 \max_j |\tilde{T}_j(\bx) - T_j(\bx)| \le 6(P-1)M_g \bar{d} e^{-cM}.$$
    The network approximation error is bounded by
    \begin{align*}
        |\hat{g}_T(\bx) - \hat{g}(\bx)| &\le \sum_{j=1}^P \tilde{\beta}_j(\bx) |\tilde{g}(\bc_j) - g(\bc_j)| + \sum_{j=1}^P |\tilde{\beta}_j(\bx) - \beta_j(\bx)| |g(\bc_j)| \\
        &\le (P-1)B e^{-cM} + 6(P-1)M_g \bar{d} e^{-cM} B \\
        &= (P-1)B (1 + 6M_g \bar{d}) e^{-cM}.
    \end{align*}
    Setting the MHA scaling factor $M = \frac{1}{c} \log \Big( \frac{2(P-1)B(1 + 6M_g \bar{d})}{\varepsilon} \Big)$ bounds this implementation error by $\varepsilon/2$. The triangle inequality yields the target bound
    \begin{equation*}
        \sup_{\bx \in \mathcal{M}} |\hat{g}_T(\bx) - g(\bx)| \le \sup_{\bx \in \mathcal{M}} |\hat{g}_T(\bx) - \hat{g}(\bx)| + \sup_{\bx \in \mathcal{M}} |\hat{g}(\bx) - g(\bx)| \le \frac{\varepsilon}{2} + \frac{\varepsilon}{2} = \varepsilon.
    \end{equation*}

    \textbf{Step 3: Parameter Complexity.}
    With the embedding dimension $D = \bar{d}+4$, similarly as the Proof of Theorem \ref{thm_holder_approx}, the total number of dense architectural parameters $\cN_{total}$ is
    \begin{align*}
        \cN_{total} &= D(P+\bar{d}+1) + 8D(P+2) + (4\bar{d}^2+35\bar{d}+76) + 4D \\
        &\quad + (4\bar{d}^2+35\bar{d}+76) + DP \\
        &= 10P(\bar{d}+4) + 9\bar{d}^2 + 95\bar{d} + 236.
    \end{align*}
    Since $P \le C_P \varepsilon^{-d/\alpha}$ and $\varepsilon \le 1/e \implies \varepsilon^{-d/\alpha} > 1$, we obtain
    \begin{equation*}
        \cN_{total} \le \big[ 10(\bar{d}+4)C_P + 9\bar{d}^2 + 95\bar{d} + 236 \big] \varepsilon^{-d/\alpha} := C_N \varepsilon^{-d/\alpha}.
    \end{equation*}
    Notice that the parameter complexity scales strictly with $\varepsilon^{-d/\alpha}$, successfully bypassing the ambient dimension $\bar{d}$ in the exponent.

    \textbf{Step 4: Parameter Bounds.}
    To bound the global parameter magnitude that $M_{max} \le \max\big\{1, M, \frac{1}{1 - 2Pe^{-cM}}, (1 + (P-1)e^{-cM}) B_{\cA_1} \big\}$, we explicitly substitute $P \le C_P \varepsilon^{-d/\alpha}$ and $M_g \le C_M \varepsilon^{-2/\alpha} \log(2/\varepsilon)$ for $\varepsilon \le 1/e$. First, we bound $M$, since
    \begin{equation*}
        \frac{1}{c} = \frac{1}{2\sin^2(\pi/P)} \le \frac{P^2}{8} \le \frac{C_P^2}{8} \varepsilon^{-2d/\alpha},
    \end{equation*}
    and
    \begin{align*}
        \log \Big( \frac{2(P-1)B(1 + 6 M_g \bar{d})}{\varepsilon} \Big) &\le \log \left( \frac{2 C_P B \varepsilon^{-d/\alpha} \big(1 + 6 \bar{d} C_M \varepsilon^{-2/\alpha} \log(2/\varepsilon)\big)}{\varepsilon} \right) \\
        &\le \log \left( \frac{2 C_P B (1 + 6 \bar{d} C_M)}{\varepsilon^{(d+2+\alpha)/\alpha}} \log\frac{2}{\varepsilon} \right) \le C_{log} \log \frac{1}{\varepsilon},
    \end{align*}
    where $C_{log} = \big| \log(2 C_P B (1 + 6 \bar{d} C_M)) \big| + \frac{d+2+\alpha}{\alpha} + 2$. It follows that
    \begin{equation*}
        M = \frac{1}{c} \log \Big( \frac{2(P-1)B(1 + 6 M_g \bar{d})}{\varepsilon} \Big) \le C_{mag} \varepsilon^{-2d/\alpha} \log \frac{1}{\varepsilon}, \quad \text{where } C_{mag} = \frac{C_P^2 C_{log}}{8}.
    \end{equation*}

    Next, by definition $e^{-cM} = \frac{\varepsilon}{2(P-1)B(1 + 6 M_g \bar{d})}$. Since $1 + 6M_g d > 2 \geq \frac{P}{P-1}$, we bound the denominator term purely in $\varepsilon$:
    \begin{equation*}
        2P e^{-cM} = \frac{2P}{P-1} \frac{\varepsilon}{2B(1 + 6 M_g \bar{d})} \le \frac{\varepsilon}{B} \implies \frac{1}{1 - 2Pe^{-cM}} \le \frac{1}{1 - \frac{1}{eB}}.
    \end{equation*}

    For the feature magnitude, substituting $M_g$ yields $B_{\cA_1} \le \max\{3M_g \bar{d}, B, 1\} \le C_B \varepsilon^{-2/\alpha} \log(2/\varepsilon)$ with $C_B = \max\{3\bar{d} C_M, B, 1\}$. Similarly bounding $(P-1)e^{-cM} \le \frac{\varepsilon}{2B}$, it follows that
    \begin{equation*}
        (1 + (P-1)e^{-cM}) B_{\cA_1} \le \Big(1 + \frac{\varepsilon}{2B}\Big) C_B \varepsilon^{-2/\alpha} \log \frac{2}{\varepsilon} \le 2 \Big(1 + \frac{1}{2eB}\Big) C_B \varepsilon^{-2d/\alpha} \log \frac{1}{\varepsilon}.
    \end{equation*}

    Consequently, the global parameter magnitude satisfies
    \begin{align} \label{tildeCmag2}
        \nonumber M_{max} &\le \max \bigg\{ \frac{1}{1 - \frac{1}{eB}}, C_{mag} \varepsilon^{-2d/\alpha} \log \frac{1}{\varepsilon}, 2 \Big(1 + \frac{1}{2eB}\Big) C_B \varepsilon^{-2d/\alpha} \log \frac{1}{\varepsilon} \bigg\} \nonumber \\
        &\le \max\bigg\{ \frac{1}{1 - \frac{1}{eB}}, C_{mag}, 2 \Big(1 + \frac{1}{2eB}\Big) C_B \bigg\} \varepsilon^{-2d/\alpha} \log \frac{1}{\varepsilon} := \widetilde{C}_{mag} \varepsilon^{-2d/\alpha} \log \frac{1}{\varepsilon},
    \end{align}
    This completes the proof.
\end{proof}

\section{Proof of generalization results in Theorems \ref{thm_gen_cube} and \ref{thm_gen_manifold}
} \label{sec:appendixC}
This final section details the proofs for the statistical generalization bounds of the empirical risk minimizer. We calculate the capacity of the hypothesis space and leverage it to derive near minimax-optimal learning rates.

\subsection{Bounding the covering number} \label{sec:covering}
We first rigorously bound the covering number of our constructed Transformer hypothesis space in the following lemma. This provides the essential capacity measure necessary for controlling the statistical estimation error.

\begin{lemma}[Covering Number of the Transformer Class] \label{lemma_covering_number}
    Let $\mathcal{T}$ be the class of Transformer networks defined in Theorem \ref{thm_holder_approx} (Theorem \ref{thm_manifold_approx}) with $L=2$, sequence length $P$, embedding dimension $D=d+4$ ($D=\bar{d}+4$), and parameter magnitude bounded by $M_{max} \ge 1$. For any $\eta \in (0, 1]$, the covering number of $\mathcal{T}$ satisfies
    \begin{equation*}
        \log \mathcal{N}(\eta, \mathcal{T}, \|\cdot\|_\infty) \le \mathcal{N}_{total} \log \left( \frac{1224256 P^4 D^{22} M_{max}^{26}}{\eta} \right),
    \end{equation*}
    where $\mathcal{N}_{total}$ is the total number of parameters.
\end{lemma}

\begin{proof}[Proof of Lemma \ref{lemma_covering_number}]
    In the proof, we consider the class of Transformer networks defined in Theorem \ref{thm_manifold_approx}, the one defined in Theorem \ref{thm_holder_approx} follows the same result. Let $\hat{g}_T(\bx; \boldsymbol{\theta})$ and $\tilde{\hat{g}}_T(\bx; \tilde{\boldsymbol{\theta}})$ be two Transformer networks in $\mathcal{T}$ parameterized by $\boldsymbol{\theta}$ and $\tilde{\boldsymbol{\theta}}$ respectively, where $\|\boldsymbol{\theta}\|_\infty, \|\tilde{\boldsymbol{\theta}}\|_\infty \le M_{max}$. Assume the parameter difference is bounded by $\|\boldsymbol{\theta} - \tilde{\boldsymbol{\theta}}\|_\infty \le \delta$. We aim to bound the output difference $\|\hat{g}_T - \tilde{\hat{g}}_T\|_\infty$ with respect to $\delta$.

    \textbf{Step 1: Pre-processing and First Encoder Block.}
    For any input $\bx \in \cM \subseteq [0,1]^{\bar{d}}$, the pre-processing layer maps $\bx$ to the initial embedding matrix $\bZ_0 \in \mathbb{R}^{D \times P}$ via
    \begin{equation*}
        \bZ_0 = \cP(\bx) = (\bW_E \bx + \bb_E) \be_1^\top + \bP,
    \end{equation*}
    where $\bW_E \in \mathbb{R}^{D \times \bar{d}}$, $\bb_E \in \mathbb{R}^D$, and $\bP \in \mathbb{R}^{D \times P}$ are parameterized by $\boldsymbol{\theta}$. Since $D=\bar d+4$, $\|\bx\|_\infty \le 1$, and all trainable parameters are bounded by $M_{max}$, we calculate the bound for the magnitude $r_0 := \|\bZ_0\|_{max}$ and the difference $\|\bZ_0 - \tilde{\bZ}_0\|_{max}$ as
    \begin{align*}
        & r_0 = \|(\bW_E \bx + \bb_E) \be_1^\top + \bP\|_{max} \le \bar{d} \|\bW_E\|_{max} \|\bx\|_\infty + \|\bb_E\|_{max} + \|\bP\|_{max} \le D M_{max}, \\
        & \|\bZ_0 - \tilde{\bZ}_0\|_{max} = \|((\bW_E - \tilde{\bW}_E) \bx + (\bb_E - \tilde{\bb}_E)) \be_1^\top + (\bP - \tilde{\bP})\|_{max} \le \bar{d} \delta + \delta + \delta \le D \delta.
    \end{align*}

    Let $\mathcal{A}_1$ and $\mathcal{F}_1$ denote the MHA and FFN mappings in the first block. We first consider the concatenated output of the $H_1$ attention heads before the final projection, denoted as $\mathcal{H}_1(\bZ_0) \in \RR^{(H_1 d_v) \times P}$. For a single attention head $h \in [H_1]$, let $\bE_1^h = \bZ_0^\top (\bK_1^h)^\top \bQ_1^h \bZ_0 \in \RR^{P \times P}$ be the pre-softmax score matrix, $\bA_1^h = \softmax(\bE_1^h) \in \RR^{P \times P}$ be the attention probability matrix, and ${\head}_1^h = \bV_1^h \bZ_0 \bA_1^h \in \RR^{d_v \times P}$ be the head output. By Theorem \ref{thm_manifold_approx}, the inner dimensions are $d_k = d_v = 2$. We decouple the difference into a parameter error and an input error
    \begin{equation*}
        \|\mathcal{H}_1(\bZ_0) - \tilde{\mathcal{H}}_1(\tilde{\bZ}_0)\|_{max} \le \underbrace{\|\mathcal{H}_1(\bZ_0) - \tilde{\mathcal{H}}_1(\bZ_0)\|_{max}}_{\text{(I) Parameter Error}} + \underbrace{\|\tilde{\mathcal{H}}_1(\bZ_0) - \tilde{\mathcal{H}}_1(\tilde{\bZ}_0)\|_{max}}_{\text{(II) Input Error}}.
    \end{equation*}
    
    \textit{(I) Parameter Error:} Fixing the input $\bZ_0$, for the pre-softmax matrix $\bE_1^h$, we have
    \begin{align*}
        \|(\bK_1^h)^\top \bQ_1^h - (\tilde{\bK}_1^h)^\top \tilde{\bQ}_1^h\|_{max} &\le d_k \|\bK_1^h\|_{max} \|\bQ_1^h - \tilde{\bQ}_1^h\|_{max} + d_k \|\bK_1^h - \tilde{\bK}_1^h\|_{max} \|\tilde{\bQ}_1^h\|_{max} \\
        & \le 2 d_k M_{max} \delta \le 4 M_{max} \delta, \\
        \|\bE_1^h - \tilde{\bE}_1^h\|_{max} &= \|\bZ_0^\top ((\bK_1^h)^\top \bQ_1^h - (\tilde{\bK}_1^h)^\top \tilde{\bQ}_1^h) \bZ_0\|_{max} \\
        & \le D^2 \|\bZ_0\|_{max}^2 \|(\bK_1^h)^\top \bQ_1^h - (\tilde{\bK}_1^h)^\top \tilde{\bQ}_1^h\|_{max} \le 4 D^2 r_0^2 M_{max} \delta.
    \end{align*}
    By Lemma \ref{Softmaxlipchitz}, for each column $j \in [P]$,
    \begin{equation*}
        \|(\bA_1^h)_{:, j} - (\tilde{\bA}_1^h)_{:, j}\|_1 \le 2 \|(\bE_1^h)_{:, j} - (\tilde{\bE}_1^h)_{:, j}\|_\infty \le 2 \|\bE_1^h - \tilde{\bE}_1^h\|_{max} \le 8 D^2 r_0^2 M_{max} \delta.
    \end{equation*}
    Since $\|(\bA_1^h)_{:, j}\|_1 = 1$, $\|\bZ_0 \bA_1^h\|_{max} \le \|\bZ_0\|_{max} \le r_0$, it follows that
    \begin{align*}
        &\|{\head}_1^h - {\tilde{\head}}_1^h\|_{max} \\
        &= \|(\bV_1^h - \tilde{\bV}_1^h) \bZ_0 \bA_1^h + \tilde{\bV}_1^h \bZ_0 (\bA_1^h - \tilde{\bA}_1^h)\|_{max} \\
        &\le D \|\bV_1^h - \tilde{\bV}_1^h\|_{max} \|\bZ_0 \bA_1^h\|_{max} + \|\tilde{\bV}_1^h \bZ_0\|_{max} \max_{j \in [P]} \|(\bA_1^h)_{:, j} - (\tilde{\bA}_1^h)_{:, j}\|_1 \\
        &\le D \delta r_0 + (D M_{max} r_0) (8 D^2 r_0^2 M_{max} \delta) \\
        &\le 9 D^3 M_{max}^2 r_0^3 \delta.
    \end{align*}
    Since $\mathcal{H}_1(\bZ_0)$ concatenates all $H_1$ heads, we have
    $$(I) \le 9 D^3 M_{max}^2 r_0^3 \delta.$$
    
    \textit{(II) Input Error:} Fixing parameters $\tilde{\boldsymbol{\theta}}$, We have
    \begin{align*}
        &\|\bZ_0^\top (\tilde{\bK}_1^h)^\top \tilde{\bQ}_1^h \bZ_0 - \tilde{\bZ}_0^\top (\tilde{\bK}_1^h)^\top \tilde{\bQ}_1^h \tilde{\bZ}_0\|_{max} \\
        &\le \|(\bZ_0 - \tilde{\bZ}_0)^\top (\tilde{\bK}_1^h)^\top \tilde{\bQ}_1^h \bZ_0\|_{max} + \|\tilde{\bZ}_0^\top (\tilde{\bK}_1^h)^\top \tilde{\bQ}_1^h (\bZ_0 - \tilde{\bZ}_0)\|_{max} \\
        &\le 2 D \|\bZ_0 - \tilde{\bZ}_0\|_{max} \|(\tilde{\bK}_1^h)^\top \tilde{\bQ}_1^h \bZ_0\|_{max} \\
        &\le 2 D \|\bZ_0 - \tilde{\bZ}_0\|_{max} (D d_k M_{max}^2 r_0) \le 4 D^2 M_{max}^2 r_0 \|\bZ_0 - \tilde{\bZ}_0\|_{max}.
    \end{align*}
    Notice that $\max_j \|(\bA_1^h)_{:, j} - (\tilde{\bA}_1^h)_{:, j}\|_1 \le 8 D^2 M_{max}^2 r_0 \|\bZ_0 - \tilde{\bZ}_0\|_{max}$. It follows that
    \begin{align*}
        &\|\tilde{\bV}_1^h \bZ_0 \bA_1^h - \tilde{\bV}_1^h \tilde{\bZ}_0 \tilde{\bA}_1^h\|_{max} \\
        &= \|\tilde{\bV}_1^h (\bZ_0 - \tilde{\bZ}_0) \bA_1^h + \tilde{\bV}_1^h \tilde{\bZ}_0 (\bA_1^h - \tilde{\bA}_1^h)\|_{max} \\
        &\le D \|\tilde{\bV}_1^h\|_{max} \|\bZ_0 - \tilde{\bZ}_0\|_{max} + \|\tilde{\bV}_1^h \tilde{\bZ}_0\|_{max} \max_{j} \|(\bA_1^h)_{:, j} - (\tilde{\bA}_1^h)_{:, j}\|_1 \\
        &\le D M_{max} \|\bZ_0 - \tilde{\bZ}_0\|_{max} + (D M_{max} r_0) (8 D^2 M_{max}^2 r_0 \|\bZ_0 - \tilde{\bZ}_0\|_{max}) \\
        &\le 9 D^3 M_{max}^3 r_0^2 \|\bZ_0 - \tilde{\bZ}_0\|_{max}.
    \end{align*}
    Thus, we have
    $$(II) \le 9 D^3 M_{max}^3 r_0^2 \|\bZ_0 - \tilde{\bZ}_0\|_{max}.$$
    
    Combining (I) and (II) with $r_0 \le D M_{max}$ and $\|\bZ_0 - \tilde{\bZ}_0\|_{max} \le D \delta$, we obtain
    \begin{align*}
        \|\mathcal{H}_1(\bZ_0) - \tilde{\mathcal{H}}_1(\tilde{\bZ}_0)\|_{max} &\le 9 M_{max}^2 D^3 r_0^3 \delta + 9 M_{max}^3 D^3 r_0^2 \|\bZ_0 - \tilde{\bZ}_0\|_{max} \\
        &\le 9 M_{max}^2 D^3 (D M_{max})^3 \delta + 9 M_{max}^3 D^3 (D M_{max})^2 (D \delta) \\
        &\le 18 D^6 M_{max}^5 \delta.
    \end{align*}
    Also, $\|\mathcal{H}_1(\bZ_0)\|_{max} \le D M_{max} r_0 \le D^2 M_{max}^2$. 
    
    For $\mathcal{A}_1(\bZ_0) = \bW^O_1 \mathcal{H}_1(\bZ_0)$, since $H_1 d_v = (P+2) \times 2 \le 6P$, we have
    \begin{align*}
        \|\mathcal{A}_1(\bZ_0) - \tilde{\mathcal{A}}_1(\tilde{\bZ}_0)\|_{max} &\le \|\bW^O_1 (\mathcal{H}_1(\bZ_0) - \tilde{\mathcal{H}}_1(\tilde{\bZ}_0))\|_{max} + \|(\bW^O_1 - \tilde{\bW}^O_1) \tilde{\mathcal{H}}_1(\tilde{\bZ}_0)\|_{max} \\
        &\le (6P) M_{max} \|\mathcal{H}_1(\bZ_0) - \tilde{\mathcal{H}}_1(\tilde{\bZ}_0)\|_{max} + (6P) \delta \|\tilde{\mathcal{H}}_1(\tilde{\bZ}_0)\|_{max} \\
        &\le 6 P M_{max} (18 D^6 M_{max}^5 \delta) + 6 P \delta (D^2 M_{max}^2) \\
        &\le 114 P D^6 M_{max}^6 \delta,
    \end{align*}
    and $\|\mathcal{A}_1(\bZ_0)\|_{max} \le (6P) M_{max} \|\mathcal{H}_1(\bZ_0)\|_{max} \le 6P D^2 M_{max}^3$. 
    
    For the FFN layer, let $\bY_1 = \sigma(\bW_1^1 \mathcal{A}_1(\bZ_0) + \bb_1^1)$ and $\tilde{\bY}_1 = \sigma(\tilde{\bW}_1^1 \tilde{\mathcal{A}}_1(\tilde{\bZ}_0) + \tilde{\bb}_1^1)$. Since $\sigma$ is $1$-Lipschitz, we have
    \begin{align*}
        \|\bY_1 - \tilde{\bY}_1\|_{max} &\le \|(\bW_1^1 - \tilde{\bW}_1^1) \mathcal{A}_1(\bZ_0) + \tilde{\bW}_1^1 (\mathcal{A}_1(\bZ_0) - \tilde{\mathcal{A}}_1(\tilde{\bZ}_0)) + (\bb_1^1 - \tilde{\bb}_1^1)\|_{max} \\
        &\le D \delta \|\mathcal{A}_1(\bZ_0)\|_{max} + D M_{max} \|\mathcal{A}_1(\bZ_0) - \tilde{\mathcal{A}}_1(\tilde{\bZ}_0)\|_{max} + \delta \\
        &\le D \delta (6P D^2 M_{max}^3) + D M_{max} (114 P D^6 M_{max}^6 \delta) + \delta \\
        &\le 121 P D^7 M_{max}^7 \delta.
    \end{align*}
    The magnitude is bounded by
    \begin{align*}
        \|\bY_1\|_{max} \le D M_{max} \|\mathcal{A}_1(\bZ_0)\|_{max} + M_{max} \le D M_{max} (6P D^2 M_{max}^3) + M_{max} \le 7 P D^3 M_{max}^4.
    \end{align*}
    Let $\bZ_1 = \mathcal{F}_1(\mathcal{A}_1(\bZ_0))$ and $\tilde{\bZ}_1 = \tilde{\mathcal{F}}_1(\tilde{\mathcal{A}}_1(\tilde{\bZ}_0))$. Since $d_{\text{ff}}^1 \le 2D$, it follows that the output difference of the first encoder block is bounded by
    \begin{align*}
        \|\bZ_1 - \tilde{\bZ}_1\|_{max} &= \|(\bW_1^2 \bY_1 + \bb_1^2) - (\tilde{\bW}_1^2 \tilde{\bY}_1 + \tilde{\bb}_1^2)\|_{max} \\
        &\le \|(\bW_1^2 - \tilde{\bW}_1^2) \bY_1 + \tilde{\bW}_1^2 (\bY_1 - \tilde{\bY}_1) + (\bb_1^2 - \tilde{\bb}_1^2)\|_{max} \\
        &\le d_{\text{ff}}^1 \delta (7 P D^3 M_{max}^4) + d_{\text{ff}}^1 M_{max} (121 P D^7 M_{max}^7 \delta) + \delta \\
        &\le 257 P D^8 M_{max}^8 \delta.
    \end{align*}
    Its magnitude is bounded by
    \begin{align*}
        r_1 := \|\bZ_1\|_{max} &\le d_{\text{ff}}^1 \|\bW_1^2\|_{max} \|\bY\|_{max} + \|\bb_1^2\|_{max} \\
        &\le 2D M_{max} (7 P D^3 M_{max}^4) + M_{max} \\
        &\le 15 P D^4 M_{max}^5.
    \end{align*}

    \textbf{Step 2: Second Encoder Block.}
    We recursively apply the same logic to the second block, utilizing $\bZ_1$ as the input bounded by $r_1$. For the second block, since $H_2=1$ and the inner dimensions are $d_k^2=d_v^2=1$, the MHA bound simplifies. Adapting the derivation from Step 1 by substituting $r_0$ with $r_1$, $\bZ_0$ with $\bZ_1$, and updating the constants for $d_k=1$, we have
    \begin{align*}
        &\|\mathcal{H}_2(\bZ_1) - \tilde{\mathcal{H}}_2(\tilde{\bZ}_1)\|_{max} \\
        &\le 5 M_{max}^2 D^3 r_1^3 \delta + 5 M_{max}^3 D^3 r_1^2 \|\bZ_1 - \tilde{\bZ}_1\|_{max} \\
        &\le 5 M_{max}^2 D^3 (15 P D^4 M_{max}^5)^3 \delta + 5 M_{max}^3 D^3 (15 P D^4 M_{max}^5)^2 (257 P D^8 M_{max}^8 \delta) \\
        &\le 16875 P^3 D^{15} M_{max}^{17} \delta + 289125 P^3 D^{19} M_{max}^{21} \delta \\
        &\le 306000 P^3 D^{19} M_{max}^{21} \delta.
    \end{align*}
    Its magnitude is bounded by $\|\mathcal{H}_2(\bZ_1)\|_{max} \le D M_{max} r_1 \le 15 P D^5 M_{max}^6$.
    
    Applying the output projection $\bW^O_2 \in \RR^{D \times 1}$ (since $H_2 d_v^2 = 1$), we have
    \begin{align*}
        \|\mathcal{A}_2(\bZ_1) - \tilde{\mathcal{A}}_2(\tilde{\bZ}_1)\|_{max} &\le 1 \cdot M_{max} \|\mathcal{H}_2(\bZ_1) - \tilde{\mathcal{H}}_2(\tilde{\bZ}_1)\|_{max} + 1 \cdot \delta \|\tilde{\mathcal{H}}_2(\tilde{\bZ}_1)\|_{max} \\
        &\le M_{max} (306000 P^3 D^{19} M_{max}^{21} \delta) + \delta (15 P D^5 M_{max}^6) \\
        &\le 306015 P^3 D^{19} M_{max}^{22} \delta.
    \end{align*}
    The magnitude is $\|\mathcal{A}_2(\bZ_1)\|_{max} \le M_{max} \|\mathcal{H}_2(\bZ_1)\|_{max} \le 15 P D^5 M_{max}^7$.
    
    For the FFN layer, we define  $\bY_2 = \sigma(\bW_2^1 \mathcal{A}_2(\bZ_1) + \bb_2^1)$ and $\tilde{\bY}_2 = \sigma(\tilde{\bW}_2^1 \tilde{\mathcal{A}}_2(\tilde{\bZ}_1) + \tilde{\bb}_2^1)$, we have
    \begin{align*}
        \|\bY_2\|_{max} &\le D M_{max} \|\mathcal{A}_2(\bZ_1)\|_{max} + M_{max} \le 16 P D^6 M_{max}^8, \\
        \|\bY_2 - \tilde{\bY}_2\|_{max} &\le D \delta \|\mathcal{A}_2(\bZ_1)\|_{max} + D M_{max} \|\mathcal{A}_2(\bZ_1) - \tilde{\mathcal{A}}_2(\tilde{\bZ}_1)\|_{max} + \delta \\
        &\le 306031 P^3 D^{20} M_{max}^{23} \delta.
    \end{align*}
    Let $\bZ_2 = \mathcal{F}_2(\mathcal{A}_2(\bZ_1))$ and $\tilde{\bZ}_2 = \tilde{\mathcal{F}}_2(\tilde{\mathcal{A}}_2(\tilde{\bZ}_1))$. Since $d_{\text{ff}}^2 \le 2D$, it follows that
    \begin{align*}
        \|\bZ_2 - \tilde{\bZ}_2\|_{max} &\le 2D \delta \|\bY_2\|_{max} + 2D M_{max} \|\bY_2 - \tilde{\bY}_2\|_{max} + \delta \\
        &\le 2D \delta (16 P D^6 M_{max}^8) + 2D M_{max} (306031 P^3 D^{20} M_{max}^{23} \delta) + \delta \\
        &\le 612095 P^3 D^{21} M_{max}^{24} \delta.
    \end{align*}
    The magnitude of the final hidden state is bounded by 
    \begin{equation*}
        r_2 := \|\bZ_2\|_{max} \le 2D M_{max} \|\bY_2\|_{max} + M_{max} \le 33 P D^7 M_{max}^9.
    \end{equation*}

    \textbf{Step 3: Readout Layer and Covering Number.}
    The final output is computed by the linear readout vector $c \in \mathbb{R}^{PD}$, where $\|c\|_\infty \le M_{max}$. The network approximation difference is bounded by
    \begin{align*}
        \|\hat{g}_T(\bx; \boldsymbol{\theta}) - \tilde{\hat{g}}_T(\bx; \tilde{\boldsymbol{\theta}})\|_\infty &\le \|c\|_1 \|\bZ_2 - \tilde{\bZ}_2\|_{max} + \|c - \tilde{c}\|_1 \|\tilde{\bZ}_2\|_{max} \\
        &\le (P D M_{max}) (612095 P^3 D^{21} M_{max}^{24} \delta) + (P D \delta) (33 P D^7 M_{max}^9) \\
        &\le 612128 P^4 D^{22} M_{max}^{25} \delta.
    \end{align*}
    This establishes the global Lipschitz constant $L_{param} = 612128 P^4 D^{22} M_{max}^{25}$ with respect to the parameters. 

    To obtain an $\eta$-covering of $\mathcal{T}$ in the function space with respect to the $\|\cdot\|_\infty$ norm, it suffices to construct a $\delta$-covering of the parameter space $\Theta = [-M_{max}, M_{max}]^{\mathcal{N}_{total}}$, where $\delta = \eta / L_{param}$. The covering number of the hypercube $\Theta$ is bounded by $(2 M_{max} / \delta)^{\mathcal{N}_{total}}$. Taking the logarithm yields
    \begin{align*}
        \log \mathcal{N}(\eta, \mathcal{T}, \|\cdot\|_\infty) &\le \mathcal{N}_{total} \log \left( \frac{2 M_{max} L_{param}}{\eta} \right) = \mathcal{N}_{total} \log \left( \frac{1224256 P^4 D^{22} M_{max}^{26}}{\eta} \right).
    \end{align*}
    This completes the proof.
\end{proof}

\subsection{Deriving learning rates of the ERM algorithm} \label{sec:learningrate}

To establish the generalization bound, we first introduce the following oracle inequality for the ERM algorithm defined over a compact subset of continuous functions.

\begin{lemma}[\cite{chui2019deep}, \cite{mao2021theory}] \label{lemma_oracle}
    Suppose there exist constants $n', \nu > 0$ such that the covering number of the hypothesis space $\mathcal{T}$ satisfies
    \begin{equation*}
        \log \mathcal{N}(\eta, \mathcal{T}, \|\cdot\|_\infty) \le n' \log \frac{\nu}{\eta}, \quad \forall \eta > 0.
    \end{equation*}
    Then for any $h^* \in \mathcal{T}$, and any $\eta > 0$, we have
    \begin{align*}
        \text{Prob}\Big\{ \|\pi_B f_{\mathcal{S}} - f_\rho\|_\rho^2 > \eta + 2\|h^* - f_\rho\|_\rho^2 \Big\} &\le \exp\left\{n' \log \frac{16\nu B}{\eta} - \frac{3n\eta}{512B^2}\right\} \\
        &\quad + \exp\left\{\frac{-3n\eta^2}{16(3B+\|h^*\|_\infty)^2(6\|h^* - f_\rho\|_\rho^2 + \eta)}\right\}.
    \end{align*}
\end{lemma}

Utilizing the covering number bounds Lemma \ref{lemma_covering_number} established above, we complete the proofs for Theorem \ref{thm_gen_cube} and Theorem \ref{thm_gen_manifold}. We optimize the bias-variance trade-off via the oracle inequality Lemma \ref{lemma_oracle} to achieve the final convergence rates.

\begin{proof}[Proof of Theorem \ref{thm_gen_cube}]
    By Theorem \ref{thm_holder_approx}, for any $\varepsilon \in (0,1/e]$, there exists a Transformer $h \in \mathcal{T}$ such that $\|h - f_\rho\|_\infty \le \varepsilon$. Since $\|f_\rho\|_\infty \le B$, we have $\|h\|_\infty \le B+\varepsilon \le B+1$.
    By Lemma \ref{lemma_covering_number}, we have
    \begin{equation*}
        \log \mathcal{N}(\eta, \mathcal{T}, \|\cdot\|_\infty) \le \mathcal{N}_{total} \log \left( \frac{1224256 P^4 D^{22} M_{max}^{26}}{\eta} \right).
    \end{equation*}
    Substituting $\mathcal{N}_{total} \le C_N \varepsilon^{-d/\alpha}$, $P \le C_P \varepsilon^{-d/\alpha}$, $D = d+4$, and $M_{max} \le \widetilde{C}_{mag} \varepsilon^{-2d/\alpha} \log \frac{1}{\varepsilon}$, we set $n' = C_N \varepsilon^{-d/\alpha}$ and obtain
    \begin{equation*}
        \nu = 1224256 (C_P \varepsilon^{-d/\alpha})^4 (d+4)^{22} \Big( \widetilde{C}_{mag} \varepsilon^{-2d/\alpha} \log \frac{1}{\varepsilon} \Big)^{26} \le C_{\nu} \varepsilon^{-56d/\alpha} \Big(\log \frac{1}{\varepsilon}\Big)^{26},
    \end{equation*}
    where $C_\nu = 1224256  C_P^4  (d+4)^{22}  \widetilde{C}_{mag}^{26}$.
    
    Applying Lemma \ref{lemma_oracle} with $h^* = h$ and $\eta \ge 6\varepsilon^2$, we have
    \begin{align*}
        \text{Prob}\Big\{ \|\pi_B f_{\mathcal{S}} - f_\rho\|_\rho^2 > 2\eta \Big\} &\le \exp\left\{ C_N \varepsilon^{-d/\alpha} \log \frac{16\nu B}{\eta} - \frac{3n\eta}{512B^2} \right\}  \\
        &+ \exp\left\{ \frac{-3n\eta^2}{16(4B+1)^2(6\varepsilon^2 + \eta)} \right\}.
    \end{align*}
    For $\eta \ge 6\varepsilon^2$, since $\varepsilon \le 1/e$, we have $\log \frac{16\nu B}{\eta} \le \left( \log\frac{8 B C_\nu}{3} + 28 + \frac{56d}{\alpha} \right) \log \frac{1}{\varepsilon}$. If we further restrict
    \begin{equation}\label{etacondition}
        \eta \ge \max\left\{ 6\varepsilon^2, C_1 \frac{\varepsilon^{-d/\alpha} \log \frac{1}{\varepsilon}}{n} \right\},
    \end{equation}
    where $C_1 = \frac{1024 B^2 C_N}{3} \left( \log\frac{8 B C_\nu}{3} + 28 + \frac{56d}{\alpha} \right)$, then we have
    \begin{eqnarray*}
        && \text{Prob} \Big\{ \|\pi_B f_{\mathcal{S}} - f_\rho\|_\rho^2 > 2\eta \Big\} \\
        &\le& \exp\left\{ C_N \varepsilon^{-d/\alpha} \log \frac{16\nu B}{\eta} - \frac{3n\eta}{512B^2} \right\} + \exp\left\{ \frac{-3n\eta^2}{16(4B+1)^2(6\varepsilon^2 + \eta)} \right\} \\
        &\le& \exp\left\{ \frac{3n\eta}{1024B^2} - \frac{3n\eta}{512B^2} \right\} + \exp\left\{ \frac{-3n\eta}{32(4B+1)^2} \right\} \\
        &\le& \exp\left\{ -\frac{3n\eta}{1024B^2} \right\} + \exp\left\{ -\frac{3n\eta}{32(4B+1)^2} \right\},
    \end{eqnarray*}
    where the second inequality follows from the restriction \eqref{etacondition}.    
    Furthermore, by setting $t = 2\eta$ and defining
    $$ C_3 = \min\left\{ \frac{3}{2048B^2}, \frac{3}{64(4B+1)^2} \right\}, $$
    we arrive at
    \begin{equation}\label{boundprob}
        \text{Prob}\Big\{ \|\pi_B f_{\mathcal{S}} - f_\rho\|_\rho^2 > t \Big\} \le 2 \exp( -C_3 n t ), \qquad \forall t \ge C_2 \max\left\{ \varepsilon^2, \frac{\varepsilon^{-d/\alpha} \log \frac{1}{\varepsilon}}{n} \right\},
    \end{equation}
    where $C_2 = \max\{12, 2C_1\}$.  
    If we apply the formula for the mean of the non-negative random variable $\xi = \|\pi_B f_{\mathcal{S}} - f_\rho\|_\rho^2$:
    $$
    \mathbb{E}[\xi] = \int_0^\infty \text{Prob} [\xi > t] dt,
    $$
    we see from \eqref{boundprob} that with $\Delta := C_2 \max\left\{ \varepsilon^2, \frac{\varepsilon^{-d/\alpha} \log \frac{1}{\varepsilon}}{n} \right\}$, there holds
    \begin{equation} \label{C4}
    \begin{aligned}
        \mathbb{E} \left[ \|\pi_B f_{\mathcal{S}} - f_\rho\|_\rho^2 \right] &= \left( \int_0^\Delta + \int_\Delta^\infty \right) \text{Prob} \Big\{ \|\pi_B f_{\mathcal{S}} - f_\rho\|_\rho^2 > t \Big\} dt \\
        &\le \Delta + \int_\Delta^\infty 2 \exp(-C_3 n t) dt \\
        &\le \Delta + \frac{2}{C_3 n} \le C_4 \max\left\{ \varepsilon^2, \frac{\varepsilon^{-d/\alpha} \log \frac{1}{\varepsilon}}{n} \right\},
    \end{aligned}
    \end{equation}
    where $C_4 = C_2 + \frac{2}{C_3}$.
    Finally, by choosing $\varepsilon = n^{-\frac{\alpha}{2\alpha+d}}$, we obtain
    \begin{equation*}
        \mathbb{E} \left[ \|\pi_B f_{\mathcal{S}} - f_\rho\|_\rho^2 \right] \le C_4 n^{-\frac{2\alpha}{2\alpha+d}} \log n,
    \end{equation*}
    Thus, we complete the proof.
\end{proof}

\begin{proof} [Proof of Theorem \ref{thm_gen_manifold}]
    The proof follows the same logic as the proof of Theorem \ref{thm_gen_cube}, utilizing the covering number bound of $\mathcal{T}$ by Lemma \ref{lemma_covering_number}, with the approximation complexities from Theorem \ref{thm_manifold_approx} depending on the intrinsic dimension $d$:
    $$ \mathcal{N}_{total} \le C_N \varepsilon^{-d/\alpha}, \quad P \le C_P \varepsilon^{-d/\alpha}, \quad M_{max} \le \widetilde{C}_{mag} \varepsilon^{-2d/\alpha} \log \frac{1}{\varepsilon}. $$
    If we restrict
    \begin{equation*}
        \eta \ge \max\left\{ 6\varepsilon^2, C_1 \frac{\varepsilon^{-d/\alpha} \log \frac{1}{\varepsilon}}{n} \right\},
    \end{equation*}
    where $C_1 = \frac{1024 B^2 C_N}{3} \left( \log\frac{8 B C_\nu}{3} + 28 + \frac{56d}{\alpha} \right)$,  $C_\nu = 1224256  C_P^4  (d+4)^{22}  \widetilde{C}_{mag}^{26}$, then we have
    \begin{equation*}
        \text{Prob} \Big\{ \|\pi_B f_{\mathcal{S}} - f_\rho\|_\rho^2 > 2\eta \Big\} \le \exp\left\{ -\frac{3n\eta}{1024B^2} \right\} + \exp\left\{ -\frac{3n\eta}{32(4B+1)^2} \right\}.
    \end{equation*}
    Furthermore, setting 
    \begin{equation} \label{C42}
        \widetilde{C}_4 = \max\{12, 2C_1\} + 2 \max\left\{ \frac{2048B^2}{3}, \frac{64(4B+1)^2}{3} \right\},
    \end{equation}
    we arrive at
    \begin{equation*}
        \mathbb{E} \left[ \|\pi_B f_{\mathcal{S}} - f_\rho\|_\rho^2 \right] \le \widetilde{C}_4 \max\left\{ \varepsilon^2, \frac{\varepsilon^{-d/\alpha} \log \frac{1}{\varepsilon}}{n} \right\}.
    \end{equation*}
    Finally, by choosing $\varepsilon = n^{-\frac{\alpha}{2\alpha+d}}$, we derive
    \begin{equation*}
        \mathbb{E} \left[ \|\pi_B f_{\mathcal{S}} - f_\rho\|_\rho^2 \right] \le \widetilde{C}_4 n^{-\frac{2\alpha}{2\alpha+d}} (\log n),
    \end{equation*}
    Thus, we complete the proof.
\end{proof}

\begingroup
\raggedbottom
\section{Numerical experiment: details}\label{app:numerical}

\subsection{Data, architecture, and training}\label{app:num-protocol}
\paragraph{Data.}
For each component $k=1,\ldots,8$, independently draw gaps $r_\ell\sim\operatorname{Unif}[0.5,1.5]$ ($\ell=1,\ldots,8$) and heights $h_\ell\sim\operatorname{Unif}[-1,1]$ ($\ell=0,\ldots,7$). Set $t_0=0$, $t_\ell=\sum_{m\le\ell}r_m/\sum_{m=1}^8r_m$, and $h_8=h_0$. The target is
\begin{equation}\label{eq:num-data}
 g_k(\bx(u))=\frac{(t_{\ell+1}-u)h_\ell+(u-t_\ell)h_{\ell+1}}{t_{\ell+1}-t_\ell},
 \quad t_\ell\le u<t_{\ell+1}.
\end{equation}
Training inputs are sampled with $u\sim\operatorname{Unif}[0,1)$.

\paragraph{Architecture.}
Using Section~\ref{sec:architecture}'s column-token convention and unscaled dot-product attention,
\begin{equation}\label{eq:num-encoder}
\begin{gathered}
 \bZ^{(0)}=[\bE\bx+\bb_E,\bm0,\ldots,\bm0]+[\bp_0,\ldots,\bp_{127}]\in\mathbb R^{16\times128},\\
 \bU=\bZ^{(0)}+\operatorname{MHA}_1(\bZ^{(0)}),\qquad
 \bZ^{(1)}=\bU+\bW_2\operatorname{ReLU}(\bW_1\bU+\bb_1\mathbf1^\mathsf T)+\bb_2\mathbf1^\mathsf T .
\end{gathered}
\end{equation}
The first block has two heads with $(d_k,d_v)=(2,2)$ and FFN width $4$; the second has one head with $(d_k,d_v)=(4,8)$. The network has $3020$ parameters. The training parameters use independent zero-mean Gaussian initialization with standard deviations:
\begin{center}
\begin{tabular}{ll@{\qquad}ll}
\toprule
Parameter & Std. & Parameter & Std.\\
\midrule
$\bE,\ \bb_E$ & $1/\sqrt2,\ 0.1$ & Tokens $\bp_j$ & $1/4$\\
First/second $\bQ,\bK$ & $1/4,\ 1/2$ & Both value projections & $1/4$\\
First/second output projections & $1/2,\ 1/\sqrt8$ & FFN weights/biases & $1/\sqrt{10},\ 0.05$\\
Readout weight/bias & $1/4,\ 0$ & &\\
\bottomrule
\end{tabular}
\end{center}

\paragraph{Optimization.}
The squared loss is
\begin{equation}\label{eq:num-loss}
 \mathcal L=\frac1{2\cdot1024\cdot8}\sum_{i=1}^{1024}\|\bbf(\bx_i)-\bg(\bx_i)\|_2^2.
\end{equation}
Adam uses $14000$ updates, batches of $256$ sampled with replacement, $(\beta_1,\beta_2)=(0.9,0.999)$, and $\epsilon=10^{-8}$, with cosine learning-rate decay from $0.003$ to $0.0003$. From the best recorded Adam state, full-batch L-BFGS uses $40$ calls with \texttt{max\_iter}$=10$, \texttt{max\_eval}$=12$, step size $1$, history $30$, strong-Wolfe search, and gradient/change tolerances $10^{-8}/10^{-11}$. Training uses float32. The final checkpoint is selected as the state with the lowest training loss among the initialization and checkpoints recorded every $100$ Adam updates and after each L-BFGS call. Calibration data do not select checkpoints. Implementation uses PyTorch~2.7.1+cu118 and NumPy~2.3.5 on a shared GTX~1070 (8 GB); recorded time for training and diagnostics was $269$ seconds.

\subsection{Evaluation and Figure~\ref{fig:numerical}}\label{app:num-metrics}
At the $N=4096$ evaluation inputs $\bx_i^{\rm test}=\bx((i-1)/N)$, let $y_{ik}=g_k(\bx_i^{\rm test})$ and $\widehat y_{ik}=h_k(\bx_i^{\rm test})$ for a predictor $\bh$. The NMSE is defined as the average of the variance-normalized mean squared prediction errors across the eight output components:
\begin{equation}\label{eq:num-nmse}
 \operatorname{NMSE}(\bh,\bg)=\frac18\sum_{k=1}^8
 \frac{\sum_{i=1}^N(\widehat y_{ik}-y_{ik})^2}{\sum_{i=1}^N(y_{ik}-\bar y_k)^2},
 \qquad \bar y_k=\frac1N\sum_{i=1}^N y_{ik}.
\end{equation}
To select active sources and order the heatmap rows, we sample $2048$ calibration inputs independently of the training data, with $u\sim\operatorname{Unif}[0,1)$. Trained attention on these inputs defines peak locations $\bc_j$ and active sources $\mathcal A$:
\begin{equation}\label{eq:num-peaks}
 \bc_j=\bx^{\rm cal}_{\arg\max_i a_j(\bx_i^{\rm cal})},\qquad
 \mathcal A=\left\{j\in\mathcal S:\frac1{2048}\sum_i a_j(\bx_i^{\rm cal})\ge\frac{0.32}{|\mathcal S|}\right\}.
\end{equation}
Figure~\ref{fig:numerical}(a,b,d) orders sources by the circle position of $\bc_j$, using the same order for both states. Panel~(c) plots the trained weights of the $25$ active sources; predictions retain all $127$ sources. Panel~(d) selects the evaluation input with the smallest maximum initial attention weight, breaking ties by the smallest $u$, giving $u=0.1553$.

\subsection{Local target approximation}\label{app:num-aux}
With initial or trained attention held fixed, we fit one affine function of $\bx$ per source to the training targets and aggregate these functions using the same attention weights:
\begin{equation}\label{eq:num-target-fit}
 (\widehat{\bB}_j,\widehat{\bb}_j)\in\arg\min_{\bB,\bb}\sum_{i=1}^{1024}a_j(\bx_i)
 \|\bB\bx_i+\bb-\bg(\bx_i)\|_2^2,\qquad
 \widetilde{\bbf}(\bx)=\sum_{j\in\mathcal S}a_j(\bx)(\widehat{\bB}_j\bx+\widehat{\bb}_j).
\end{equation}
Constant fits impose $\bB=\bm0$. Minimum-norm weighted least squares uses float64 SVD with relative cutoff $10^{-10}$. Only training data are used for these fits. Targets and predictions are evaluated on the $4096$ equally spaced circle points specified in Appendix~\ref{app:num-metrics}.

\begin{center}
\begin{minipage}{\textwidth}
\centering
\small
\captionsetup{hypcap=false}
\captionof{table}{NMSE for the network in Figure~\ref{fig:numerical}. Local fits use \eqref{eq:num-target-fit} with initial or trained attention held fixed.}
\label{tab:native-aggregate}
\begin{tabular}{lcc}
\toprule
Predictor & Initial attention & Trained attention\\
\midrule
Trained network & --- & $9.07\times10^{-4}$\\
Constant target fits & $0.551$ & $0.125$\\
Affine target fits & $0.165$ & $0.00664$\\
\bottomrule
\end{tabular}
\end{minipage}
\end{center}

Table~\ref{tab:native-aggregate} examines how the spatial structure of attention affects the aggregation of source-wise constant and affine fits. At initialization, the relatively diffuse attention weights do not induce clearly localized fitting regions. After training, the attention weights become more spatially concentrated, allowing the same simple fitting functions to serve as effective local approximants. In particular, the NMSE of the affine reconstruction decreases from $0.165$ with initial attention to $0.00664$ with trained attention. This suggests that learned softmax attention can induce a useful spatial localization structure for local-to-global approximation, consistent with the underlying principle of our Softmax POU framework.

\begin{center}
\begin{minipage}{\textwidth}
\centering
\captionsetup{hypcap=false}
\includegraphics[width=\textwidth]{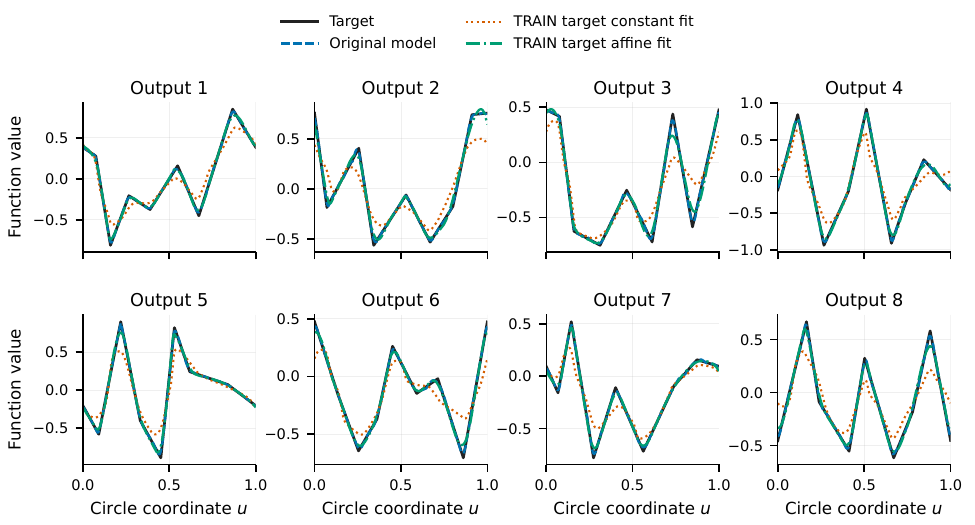}
\captionof{figure}{Eight regression outputs: targets, trained model predictions, and local constant/affine fits. Local fits use frozen trained attention.}
\label{fig:numerical-appendix}
\end{minipage}
\end{center}

Figure~\ref{fig:numerical-appendix} provides a component-wise comparison of the target functions, the trained Transformer predictions, and the attention-weighted constant and affine reconstructions. The affine fits follow the target variations more closely than the constant fits, highlighting the complementary roles of local approximation and attention-based aggregation: affine functions capture variations within individual neighborhoods, while the learned attention weights smoothly combine their contributions across the circle. This behavior is qualitatively consistent with our Softmax POU framework, in which spatially localized weights assemble local approximants into a global predictor. Furthermore, the trained Transformer achieves even lower error than the affine reconstruction, suggesting that its learned source representations capture finer local variations than the simple affine approximants.

\par\clearpage
\endgroup

\end{document}